\documentclass[12pt]{article}
\usepackage{fullpage}

\usepackage{amsmath,amsfonts,amssymb}
\usepackage{algorithm,algorithmic}
\usepackage{multirow}
\usepackage{cellspace}
\usepackage{setspace}
\usepackage{bm}
\usepackage{bbm}
\usepackage{subfigure}
\usepackage{graphicx}
\usepackage{tikz,pgfplots}
\usepackage{framed}

\usepackage[round]{natbib}

\usepackage{hyperref}
\hypersetup{colorlinks,
            linkcolor=blue,
            citecolor=blue,
            urlcolor=magenta,
            linktocpage,
            plainpages=false}

\usepackage[capitalise]{cleveref}

\newcommand{\startfoo}{%
    \par\medskip
    \begin{mdframed}[linewidth=1pt]%
    \let\figure\figurehere
    \let\endfigure\endfigurehere
    \ignorespaces
}
\newcommand{\stopfoo}{%
    \unskip
    \end{mdframed}%
    \par\medskip
}

\usepackage{amsmath,amsthm,amssymb}
\usepackage{xcolor}

\theoremstyle{plain}
\newtheorem{theorem}{Theorem}
\newtheorem{lemma}[theorem]{Lemma}
\newtheorem{corollary}[theorem]{Corollary}

\newtheorem*{theorem*}{Theorem}
\newtheorem*{lemma*}{Lemma}
\newtheorem*{corollary*}{Corollary}
\newtheorem*{proposition*}{Proposition}
\newtheorem*{claim*}{Claim}
\newtheorem*{fact*}{Fact}

\theoremstyle{definition}

\newtheorem*{definition*}{Definition}

\newtheorem*{remark*}{Remark}

\newtheorem*{example*}{Example}

\newcommand{\ignore}[1]{}

\DeclareMathOperator*{\argmin}{arg\,min}

\newcommand{\be}{\begin{align}}
\newcommand{\en}{\end{align}}

\newcommand{\ben}{\begin{align*}}
\newcommand{\enn}{\end{align*}}

\newcommand{\tr}{^{\top}}

\newcommand{\st}{\star}

\renewcommand{\t}[1]{\tilde{#1}}
\newcommand{\abs}[1]{\left|#1\right|}
\newcommand{\norm}[1]{\left\|#1\right\|}
\newcommand{\lr}[1]{\left(#1\right)}
\newcommand{\set}[1]{\left\{#1\right\}}

\newcommand{\reals}{\mathbb{R}}

\newcommand{\D}{\mathcal{D}}
\newcommand{\K}{\mathcal{K}}

\newcommand{\eps}{\varepsilon}

\renewcommand{\t}[1]{\tilde{#1}}
\newcommand{\wt}[1]{\smash{\widetilde{#1}}}
\renewcommand{\O}{O}

\newcommand{\tO}{\wt{\O}}
\newcommand{\E}{\mathbf{E}}

\newcommand{\KL}{\mathrm{KL}}

\newcommand{\half}{\frac{1}{2}}
\newcommand{\ind}[1]{1\!\!1_{#1}}
\newcommand{\ds}{\displaystyle}

\newcommand{\Blr}[1]{\Big(#1\Big)}

\newcommand{\wbar}{\overline{w}}
\newcommand{\A}{\mathcal{A}}
\newcommand{\W}{\mathcal{W}}
\renewcommand{\eps}{\epsilon}
\renewcommand{\O}{O}
\renewcommand{\st}{*}
\newcommand{\poly}{\text{poly}}

\title{Logistic Regression: Tight Bounds \\ for Stochastic and Online Optimization%
\footnote{The research leading to these results has received funding from the European Union's Seventh Framework Programme (FP7/2007-2013) under grant agreement n$^{\circ}$ 336078 -- ERC-SUBLRN.}}

\author{%
Elad Hazan\footnote{Technion---Israel Institute of Technology, Haifa 32000, Israel. 
Emails: \texttt{ehazan@ie.technion.ac.il}, \texttt{tomerk@technion.ac.il}, \texttt{kfiryl@tx.technion.ac.il}.}
\and
Tomer Koren\footnotemark[2]
\and 
Kfir Y. Levy\footnotemark[2]
}

\date{May 2014}                                           

\begin{document}
\maketitle

\begin{abstract}
The logistic loss function is often advocated in machine learning and statistics as a smooth and strictly convex surrogate for the 0-1 loss. In this paper we investigate the question of whether these smoothness and convexity properties make the logistic loss  preferable to other widely considered options such as the hinge loss. 
We show that in contrast to known asymptotic bounds, as long as the number of prediction/optimization iterations is sub exponential, the logistic loss provides no improvement over a generic non-smooth  loss function  such as the hinge loss. 
In particular we show that the convergence rate of stochastic logistic optimization is bounded from below by a polynomial in the diameter of the decision set and the number of prediction iterations, and provide a matching tight upper bound. This resolves  the COLT open problem of \cite{mcmahan2012open}.
\end{abstract}


\section{Introduction}
In many applications, such as estimation of click-through-rate in web advertising, and predicting whether a patient has a certain disease, the logistic loss is often the loss of choice. 
It appeals as a convex surrogate of the 0-1 loss, and as a tool that not only yields categorical prediction but also able to estimate the underlying probabilities of the categories. 
Moreover, \cite{friedman2000additive} and \cite{collins2002logistic} have shown that logistic regression is  strongly connected to boosting.

A long standing debate in the machine learning community has been the optimal choice of surrogate loss function for binary prediction problems (see \cite{BlogLangford}, \cite{BlogBulatov}). Amongst the arguments in support of the logistic loss are its smoothness and strict-convexity properties, which unlike other loss functions (such as the hinge loss), permit the use of more efficient optimization methods. In particular, the logistic loss is exp-concave, and thus second-order methods are applicable and give rise to theoretically superior convergence and/or regret bounds.

More technically, under standard assumptions on the training data, the logistic loss is 1-Lipschitz and $e^{-D}$-exp-concave over the set of linear $n$-dimensional classifiers whose $L_{2}$-norm is at most~$D$.
Thus, the Online Newton Step algorithm \citep{HazanAK07} can be applied to the logistic regression problem and gives a convergence rate of $\tO(e^{D} n/T)$ over $T$ iterations.
On the other hand, first order methods can be used to attain a rate of $\O(D/\sqrt{T})$, which is attainable in general for any Lipschitz convex loss function. 
The exponential dependence on~$D$ of the first bound suggests that second order methods might present poor performance in practical logistic regression problems, even when compared to the slow $1/\sqrt{T}$ rate of first-order methods.
The gap between the two rates raises the question: {\bf is a fast convergence rate of the form $\wt{\bm{O}}\bm{(\poly(D)/T)}$ achievable for logistic regression?}

This question has received much attention lately. 
\cite{bach2013adaptivity}, relying on a property called ``generalized self-concordance'', gave an algorithm with convergence rate of $\O(D^{4}/\mu^{*} T)$, where $\mu^*$ is the smallest eigenvalue of the Hessian at the optimal point.
This translates to a~$\O(\poly(D)/T)$ rate whenever the expected loss function is ``locally strongly convex'' at the optimum.
More recently, \cite{bach2013nonstrong} extended this result and presented an elegant algorithm that attains a rate of the form~$\O(\rho^{3} D^{4} n/T)$,
without assuming strong convexity (neither global or local) ---  but rather depending on a certain data-dependent constant~$\rho$.

In this paper, we resolve the above question and give tight characterization of the achievable convergence rates for logistic regression. 
We show that as long as the target accuracy $\eps$ is not exponentially small in $D$, a rate of the form $\tO(\poly(D)/T)$ is not attainable.
Specifically, we prove a lower bound of $\Omega(\sqrt{D/T})$ on the convergence rate, that can also be achieved (up to a $\sqrt{D}$ factor) by stochastic gradient descent algorithms.
In particular, this shows that in the worst case, the magnitude of data-dependent parameters used in previous works are exponentially large in the diameter $D$.
The latter lower bound only applies for multi-dimensional regression (i.e., when $n \ge 2$);
surprisingly, in one-dimensional logistic regression we find a rate of $\Theta(T^{-2/3})$ to be tight.
As far as we know, this is the first natural setting demonstrating such a phase transition in the optimal convergence rates, with respect to the dimensionality of the problem.
\setlength{\cellspacetoplimit}{3pt}
\setlength{\cellspacebottomlimit}{3pt}
\begin{table}[h!] \label{state}
\begin{center}
\begin{tabular}{ |l||Sc @{} Sl|Sc @{} Sl|Sc @{} Sl| }
\hline
\multirow{2}{*}{Setting} & \multicolumn{2}{c|}{\multirow{2}{*}{Previous}} & \multicolumn{4}{c|}{This Paper}\\ 
\cline{4-7}
& & & \multicolumn{2}{c|}{$n=1$} & \multicolumn{2}{c|}{$n \ge 2$} \\
\hline\hline
Stochastic & $\ds\O\Blr{\frac{D}{\sqrt{T}}}$~&~[\citeauthor{Zinkevich03}] & $\ds\O\Blr{\frac{D^3}{T^{2/3}}}$~&~[Cor.~\ref{Corollary:LogitFTRL}] & $\ds\Omega\Blr{\sqrt{\frac{D}{T}}}$~&~[Thm.~\ref{Theorem:2DimLowerBound}] \\
& $\ds\O\lr{\frac{e^D \log{T}}{T}}$~&~[\citeauthor{HazanAK07}] & $\ds\Omega\Blr{\frac{D^{2/3}}{T^{2/3}}}$~&~[Thm.~\ref{Theorem:1DimLowerBound}] & & \\
\hline
Online & $\ds\O(D\sqrt{T})$~&~[\citeauthor{Zinkevich03}] & $\ds\O(D^3 \, T^{1/3})$~&~[Thm.~\ref{Theorem:LogitFTRL}] & $\ds\Omega(\sqrt{DT})$~&~[Cor.~\ref{Corollary:2DimLowerBound}] \\
&  $\ds\O(e^D \log T)$~&~[\citeauthor{HazanAK07}] & $\ds\Omega(D^{2/3}\,T^{1/3})$~&~[Cor.~\ref{Corollary:1DimLowerBound}] & & \\
\hline
\end{tabular}
\end{center}
\caption{Convergence rates and regret bounds for the logistic loss, in the regime $T =  \O(e^D)$. } 
\end{table}

We also consider the closely-related online optimization setting, where on each round~$t=1,2,\ldots,T$ an adversary chooses a certain logistic function and our goal is to minimize the $T$-round regret, with respect to the best fixed decision chosen with the benefit of hindsight.
In this setting, \cite{mcmahan2012open} investigated the one-dimensional case and showed that if the adversary is restricted to pick binary (i.e.~$\pm 1$) labels, a simple follow-the-leader algorithm attains a regret bound of $\O(\sqrt{D}+\log{T})$. 
This discovery led them to conjecture that bounds of the form $\O(\poly(D)\log{T})$ should be achievable in the general multi-dimensional case with continuous labels set.

Our results extend to the online optimization setup and resolve the COLT 2012 open problem of \cite{mcmahan2012open} on the negative side.
Namely, we show that as long as the number of rounds $T$ is not exponentially large in $D$, an upper bound of $\O(\poly(D)\log{T})$ cannot be attained in general.
We obtain lower bounds on the regret of $\Omega(\sqrt{D\,T})$ in the multi-dimensional case and~$\Omega(D^{2/3}T^{1/3})$ in the one-dimensional case, when allowing the adversary to use a continuous label set. We are not aware of any other natural problem that exhibits such a dichotomy between the minimax regret rates in the one-dimensional and multi-dimensional cases.

 It is interesting to note that our bounds apply to a finite interval of time, namely when $T = \O(e^D)$, which is arguably the regime of interest for reasonable values of $D$. This is the reason our lower bounds do not contradict the logarithmic known regret bounds.

We prove the tightness of our one-dimensional lower bounds, in both the stochastic and online settings, by devising an online optimization algorithm specialized for one-dimensional online logistic regression that attains a regret of $\O(D^{3} \, T^{1/3})$.
This algorithm maintains approximations of the observed logistic loss functions, and use these approximate losses to form the next prediction by a follow-the-regularized-leader (FTRL) procedure. 
As opposed to previous works that utilize approximate losses based on \emph{local} structure~\citep{Zinkevich03,HazanAK07}, we find it necessary to employ approximations that rely on the \emph{global} structure of the logistic loss.

The rest of the paper is organized as follows. 
In \cref{section:Preliminaries} we describe the settings we consider and give the necessary background. 
We present our lowers bounds in \cref{section:LowerBounds}, and in \cref{section:UpperBounds} we prove our upper bound for one dimensional logistic regression. 
In \cref{Section:Proofs} we give complete proofs of our results.
We conclude in \cref{sec:summary}. 

\section{Setting and Background} \label{section:Preliminaries}

In this section we formalize the settings of stochastic logistic regression and online logistic regression and give the necessary background on both problems.

\subsection{Stochastic Logistic Regression}
In the problem of stochastic logistic regression, there is an unknown distribution $\mathcal{D}$ over instances $x\in \reals^{n}$.
For simplicity, we assume that $\norm{x} \le 1$.
The goal of an optimization algorithm is to minimize the expected loss of a linear predictor $w \in \reals^{n}$, 
\begin{align} \label{eq:Loss}
	L(w) 
	~=~ \E_{x \sim \D}[\, \ell(w , x) \,] ~,
\end{align}
where $\ell$ is the logistic loss function%
\footnote{The logistic loss is commonly defined as $\ell(w; x,y) = \log \big( 1+\exp(-y x \cdot w) \big)$ for instances $(x,y) \in \reals^{n} \times [-1,1]$. For ease of notation and without loss of generality, we ignore the variable $y$ in the instance $(x,y)$ by absorbing  it into~$x$. 
},
\begin{align*}
	\ell(w, x) 
	~=~ \log \big( 1+\exp(x \cdot w) \big)
\end{align*}
that expresses the negative log-likelihood of the instance $x$ under the logit model.
While we may try to optimize $L(w)$ over the entire Euclidean space, for generalization purposes we usually restrict the optimization domain to some bounded set.
In this paper, we focus on optimizing the expected loss over the set $\W = \set{w \in \reals^{n} \,:\, \norm{w} \le D}$, the Euclidean ball of radius $D$.
We define the \emph{excess loss} of a linear predictor $w \in \W$ as the difference $L(w) - \min_{w^{\st} \in \W} L(w^{\st})$ between the expected loss of $w$ and the expected loss of the best predictor in the class $\W$.

An algorithm for the stochastic optimization problem, given a sample budget $T$ as a parameter, may use a sample $x_{1}, \ldots,x_{T}$ of $T$ instances sampled independently from the distribution $\D$, and produce an approximate solution~$\wbar_T$.
The \emph{rate of convergence} of the algorithm is then defined as the expected excess loss of the predictor $\wbar_T$, given by 
$$
	\E[L(\wbar_T)] ~-~ \min_{w^{\st} \in \W} L(w^{\st}) ~,
$$ 
where the expectation is taken with respect to both the random choice of the training set and the internal randomization of the algorithm (which is allowed to be randomized).

\subsection{Online Logistic Regression}

Another optimization framework we consider is that of online logistic optimization, which we formalize as the following game between a player and an adversary.
On each round $t=1,2,\ldots,T$ of the game, the adversary first picks an instance $x_{t} \in \reals^{n}$, the player then chooses a linear predictor $w_{t} \in \W = \set{w \in \reals^{n} \,:\, \norm{w} \le D}$, observes  $x_{t}$ and incurs loss
\begin{align*}
	\ell(w_{t} , x_{t}) 
	~=~ \log \big( 1+\exp( x_{t} \cdot w_{t}) \big) ~.
\end{align*}
For simplicity we again assume that $\norm{x_{t}} \le 1$ for all $t$.
The goal of the player is to minimize his regret with respect to a fixed prediction from the set $\W$, which is defined as
\begin{align*}
	\text{Regret}_{T} 
	~=~ \sum_{t=1}^{T} \ell(w_{t}, x_{t}) 
		~-~ \min_{w^{*} \in \W} \sum_{t=1}^{T} \ell(w^\st, x_{t}) ~.
\end{align*}


\subsection{Information-theoretic Tools}

As a part of our lower bound proofs, we utilize two impossibility theorems that assert the minimal number of samples needed in order to distinguish between two distributions. 
We prove the following lower bound on the performance of any algorithm for this task.
 
%
\begin{theorem} \label{thm:coin}
Assume a coin with bias either $p$ or $p+\eps$, where $p \in (0,\half]$, is given.
Any algorithm that correctly identifies the coin's bias with probability at least $3/4$, needs no less than $p/16\eps^2$ tosses.
\end{theorem}
The theorem applies to both deterministic and randomized algorithms; in case of random algorithms the probability is with respect to both the underlying distribution of the samples, and the randomization of the algorithm. 
The proof of \cref{thm:coin} is given, for completeness, in \cref{appendixA:InformationTheoretic}.

%

\section{Lower Bounds for Logistic Regression} \label{section:LowerBounds}
In this section we derive lower bounds for the convergence rate of stochastic logistic regression. 
For clarity, we lower bound the number of observations $T$ required in order to attain excess loss of at most $\eps$, which we directly translate to a bound for the convergence rate. The stochastic optimization lower bounds are then used to obtain  corresponding  bounds for the online setting.  

In \cref{section:lowerbound1Dim} we prove a lower bound for the one dimensional case, in \cref{section:lowerbound2Dim} we prove another lower bound for the multidimensional case, and in \cref{section:LowerBoundsOnline} we present our lower bounds for the online setting.


\subsection{One-dimensional Lower Bound for Stochastic Optimization} \label{section:lowerbound1Dim}

We now show that any algorithm for one-dimensional stochastic optimization with logistic loss, must observe at least $\Omega(D/\eps^{1.5})$ instances before it provides an instance with $\eps$ expected excess loss. This directly translates to a convergence rate of $\Omega(D^{2/3}/T^{2/3})$. Formally, the main theorem of this section is the following.

\begin{theorem}\label{Theorem:1DimLowerBound}
Consider the one dimensional stochastic logistic regression setting with a fixed sample budget $T = \O(e^{D})$. 
For any algorithm $\A$ there exists a distribution $\D$ for which the expected excess loss of $\A$'s output 
 is at least $\Omega(D^{2/3}/T^{2/3})$.
\end{theorem}
%

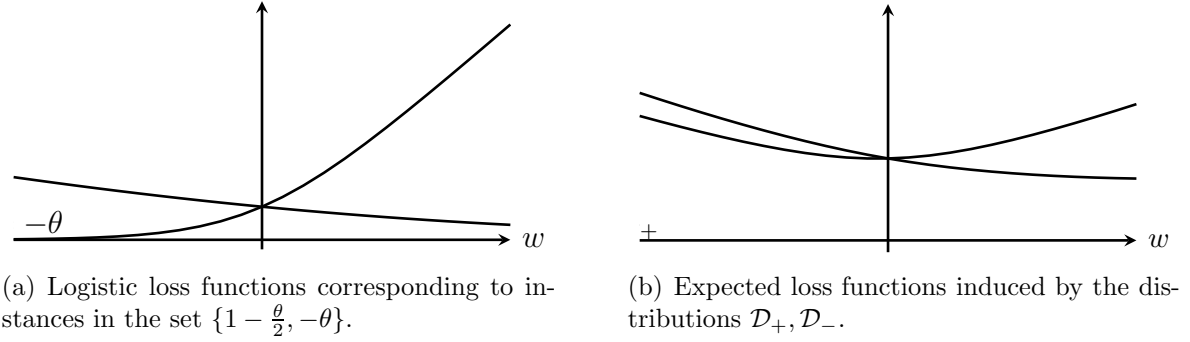
\begin{figure}[h]
\centering
\subfigure[Logistic loss functions corresponding to instances in the set $\{1-\tfrac{\theta}{2}, -\theta\}$.]{ \label{fig:losses1dim}
\begin{tikzpicture}
    \begin{axis}[
    	width=0.4\textwidth,
	height=0.2\textwidth,
	scale only axis,
	xmin=-5,xmax=5,
	ymin=-0.2,ymax=5,
	axis lines=middle, 
	xlabel={$w$},
	xlabel style={right},
	ticks=none, 
	line width=1pt]
    \addplot[line width=1pt,domain=-5:5] {ln(1+exp(x*(1-1/10)))}
    node[pos=0.95,xshift=-35pt] {$x=1-\frac{\theta}{2}$};
    \addplot[line width=1pt,domain=-5:5] {ln(1+exp(x*(-1/5)))}
    node[pos=0.15,yshift=10pt] {$x=-\theta$};
    \end{axis}
\end{tikzpicture}}
\hskip 0.05\textwidth
\subfigure[Expected loss functions induced by the distributions $\D_{+},\D_{-}$.]{\label{fig:values1dim}
\begin{tikzpicture}
    \begin{axis}[
    	width=0.4\textwidth,
	height=0.2\textwidth,
	scale only axis,
	xmin=-5,xmax=5,
	ymin=-0.1,ymax=2,
	axis lines=middle, 
	xlabel={$w$},
	xlabel style={right},
	ticks=none, 
	line width=1pt,
	font=\small]
    \addplot[line width=1pt,domain=-5:5] {
      0.05*ln(1+exp(x*(1-1/10)))
      + 0.95*ln(1+exp(x*(-1/5)))
    }
    node[pos=0.9,yshift=-8pt] {$L_{-}$};
    \addplot[line width=1pt,domain=-5:5] {
      0.20*ln(1+exp(x*(1-1/10)))
      + 0.80*ln(1+exp(x*(-1/5)))
    }
    node[pos=0.9,yshift=10pt] {$L_{+}$};
    \end{axis}
\end{tikzpicture}}
%
\caption{Loss functions used in the one-dimensional construction, and the induced expected loss functions.} \label{fig:1DimLower}
\end{figure}

The proof of \cref{Theorem:1DimLowerBound} is given at the end of this section; 
here we give an informal proof sketch.
Consider distributions $\D$ over the two-element set $\{ 1-\frac{\theta}{2},-\theta  \}$. 
For $w\in [D/2,D]$ and $\theta\ll 1$, the losses of these instances are approximately linear/quadratic with opposed slopes (see \cref{fig:losses1dim}). 
Consequently, we can build a distribution with an expected loss which is quadratic in $w$; upon perturbing the latter distribution by $\pm \eps$ we get two distributions $\D_{+},\D_{-}$ with expected losses $L_{+},L_{-}$ that are approximately linear in $w$ with slopes $\pm \eps$ (see \cref{fig:values1dim}). 
An algorithm that attains a low expected excess loss on both these distributions can be used to distinguish between them, we then utilize an information theoretic impossibility theorem  
to bound the number of observations needed in order to distinguish between two distributions.

\begin{figure}[h]
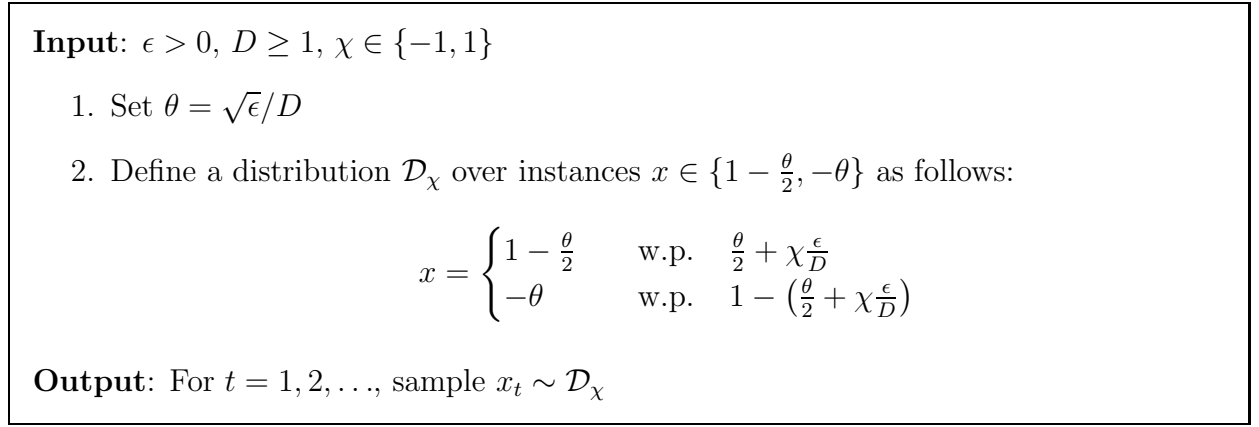

\begin{framed}
 \textbf{Input}: $\eps>0$, $D \geq 1$, $\chi\in\{-1,1\}$
\begin{enumerate}
\item Set $\theta = \sqrt{\eps}/D$
\item Define a distribution $\D_{\chi}$ over instances $x \in \{ 1-\tfrac{\theta}{2},-\theta \}$ as follows:
\begin{equation}\nonumber
x=
\begin{cases}
	1-\frac{\theta}{2} 	&\quad \text{w.p. ~~$\frac{\theta}{2}+\chi\frac{\eps}{D}$ } \\ 
	-\theta            	&\quad \text{w.p. ~~$1-\lr{\frac{\theta}{2}+\chi \frac{\eps}{D}}$}\\ 
\end{cases}
\end{equation}
\end{enumerate}
 \textbf{Output}: For $t=1,2,\ldots$, sample  $x_t\sim \D_{\chi}$
\end{framed}
\caption{Two distributions: $\D_{\chi}$, $\chi\in \{-1,1\}$; any algorithm that attains an $\eps$ expected excess logistic loss on both of them requires $\Omega(D/\eps^{1.5})$ observations.}
\label{fig:Adversary1Diml}
\end{figure}

In \cref{fig:Adversary1Diml} we present two distributions, 
which we denote by $\D_{+}$ and $\D_{-}$. 
We denote by $L_{+},L_{-}$ the expected logistic loss of a predictor $w\in \W$ with respect to $\D_{+}, \D_{-}$, i.e.,
\begin{align*}
	L_{\chi}(w) 
	&~=~ \E_{\D_{\chi}}[ \ell(w,x) ] \\
	&~=~ \lr{ \frac{\theta}{2}+\chi\frac{\eps}{D} } \ell \lr{w,1-\frac{\theta}{2}}
	+ \lr{ 1-\frac{\theta}{2}-\chi \frac{\eps}{D} } \ell \lr{w,-\theta} ~,
	\qquad \chi\in \{-1,1\} ~.
\end{align*}
The following lemma states that it is impossible attain a low expected excess loss on both $\D_{+}$ and $\D_{-}$ simultaneously.
Here we only give a sketch of the proof; the complete proof is deferred to \cref{Section:Proof:Lemma:TwoAdversaries1Dim}.
\begin{lemma} \label{Lemma:TwoAdversaries1Dim}
Given $D\geq 1$ and $\Omega(e^{-D})\leq \eps \leq 1/25$, 
consider the  distributions $\D_{+},\D_{-}$ defined in \cref{fig:Adversary1Diml}.
Then the following holds:
\begin{align*}
	L_{+}(w)-\min_{w^*\in \W }L_{+}(w^*) 
	&~\ge~  \eps /20~, \qquad \forall ~ w \in [ \tfrac{3}{4} D, D] ~,\\
	L_{-}(w)-\min_{w^*\in \W}L_{-}(w^*) 
	&~\ge~  \eps /20~, \qquad \forall ~ w \in [-D, \tfrac{3}{4} D] ~.
\end{align*}
\end{lemma}
\begin{proof}[Proof (sketch)]
First we show that for $w \in [\tfrac{1}{2}D,D]$, the losses of  the instances $1-\frac{\theta}{2},-\theta$ are approximately linear/quadratic, i.e.,
\begin{align*}
 \abs{ \ell(w,1-\tfrac{\theta}{2}) - (1-\tfrac{\theta}{2})w }
 &~\le~ \frac{\eps }{40} ~,  \qquad \forall ~ w\in [\tfrac{1}{2}D,D] ~, \\
 \abs{ \ell(w,-\theta) - \left(\log 2-\tfrac{\theta}{2}w+\tfrac{1}{8} (\theta w)^2\right) }
 &~\le~  \frac{\eps }{40} ~, \qquad \forall ~ w \in [\tfrac{1}{2}D,D] ~.
 \end{align*}
Using the above approximations  and $\theta = \sqrt{\eps}/D$, we show that $L_{+}(w)\approx \eps w/D+\eps w^2/8D^2$ and  $L_{-}(w)\approx -\eps w/D+\eps w^2/8D^2$ for $w\in[\tfrac{1}{2}D,D]$, where ``$\approx$" denotes equality up to an additive term  of $\eps/40$. 
Thus,
\begin{align*}
	L_{+}(w)-\min_{w^* \in \W}L_{+}(w^*)
		&~\ge~ L_{+}(w)-L_{+}(D/2)  
		~\ge~ {\eps }/{20} ~,
		&&\forall ~ w\in [\tfrac{3}{4}D,D] ~,\\
	L_{-}(w)-\min_{w^* \in \W}L_{-}(w^*) 
		&~\ge~ L_{-}(w)-L_{-}(D)
		~\ge~ {\eps }/{20} ~,
		&&\forall ~ w\in [\tfrac{1}{2}D,\tfrac{3}{4}D] ~.
\end{align*}
Showing that $L_{-}$ is monotonically decreasing  in $[-D, \tfrac{1}{2} D]$, extends the latter inequality to  $[-D, \tfrac{3}{4} D]$.  
\end{proof}
We are now ready to prove \cref{Theorem:1DimLowerBound}.
\begin{proof}[Proof of \cref{Theorem:1DimLowerBound}]
Consider some algorithm $\A$; we will show that if $\A$ observes $T$ samples from a distribution $\D$ which is either $\D_{+}$ or $\D_{-}$, then the expected excess loss $\tilde{\epsilon}$ that $\A$ can guarantee is lower bounded by $\Omega(D^{2/3}T^{-2/3})$.

The excess loss is non negative; therefore, 
if $\mathcal{A}$ guarantees an expected excess loss smaller than $\tilde{\epsilon}:=\eps /80$, then by Markov's inequality   
 it achieves an excess loss smaller than  $\eps /20$, w.p. $\geq 3/4$. Denoting by  $\wbar_T$ the predictor that $\A$ outputs after $T$ samples, then according to \cref{Lemma:TwoAdversaries1Dim}, attaining an excess loss smaller than $\eps /20$ on the distribution $\D_{+}$ (respectively $\D_{-}$) implies $\wbar_T \le \tfrac{3}{4} D$ (respectively $\wbar_T>\tfrac{3}{4} D$).

Since  $\A$ achieves an excess loss smaller than $\eps /20$ w.p. $\geq3/4$ for any distribution $\D$  
we can use its output  to identify the right distribution w.p. $\geq 3/4$. This can be done as follows:
\begin{align*}
\text{If    }\;  \wbar_T \leq \tfrac{3}{4} D, \;\text{ Return: ``$\D_{+}$" ;}\\
\text{If    }\;  \wbar_T > \tfrac{3}{4} D, \;\text{ Return: ``$\D_{-}$" .}
\end{align*}
According to \cref{thm:coin} distinguishing between these two distributions (``coins") w.p. $\geq3/4$ requires that the number of observations $T$ to be lower bounded as follows:
$$
	T 
	~\ge~ \frac{\theta/2-\eps/D}{16(2\eps/D)^2} 
	~\ge~ \frac{1}{256} \frac{D}{\eps^{1.5}} ~,
$$
We used $\theta/2-\eps/D$ as a lower bound on the bias of  $\D_{-}$; since $\theta = \sqrt{\eps}/D$ and $\epsilon \leq 1/25$  it follows that $\theta/2-\eps/D \geq \sqrt{\epsilon}/4D$.
We also used $2\eps/D$ as the bias between the ``coins" $\D_{+}$, $\D_{-}$. Using the above inequality together with $\tilde{\epsilon}=\epsilon/80$ yields a
lower bound of $\frac{1}{4000}{D^{2/3}}T^{-2/3}$ on the expected excess loss.
\end{proof}

\subsection{Multidimensional Lower Bound  for Stochastic Optimization} \label{section:lowerbound2Dim}
We now construct two distribution over instance vectors from the unit ball of $\reals^{2}$, and prove that any algorithm that attains an expected excess loss at most $\eps$ on both distributions requires $\Omega(D/\eps^{2})$ samples in the worst case. This directly translates to a convergence rate of $\Omega(\sqrt{D/T})$.
For $n > 2$ dimensions, we can embed the same construction in the unit ball of $\reals^{n}$, thus our bound holds in any dimension greater than one. 
The main theorem of this section is the following.
\begin{theorem}\label{Theorem:2DimLowerBound}
Consider the multidimensional stochastic logistic regression setting with $D \ge 2$ and a fixed sample budget $T = \O(e^D)$. 
For any algorithm $\A$ there exists a distribution $\D$ such that the expected excess loss of $\A$'s output is at least $\Omega(\sqrt{D/T})$.
\end{theorem}
\cref{Theorem:2DimLowerBound} is proved at the end of this section. We bring here an informal description of the proof:

\begin{figure}[h]
\centering
\begin{tikzpicture}[scale=1.85, line width=1pt]
  \fill[gray!40] ++ (110:1) arc (110:-20:1);
  \fill[gray!40] ++ (70:1) arc (70:200:1);
  \begin{scope}
    \clip ++ (110:1) arc (110:-20:1);
    \clip ++ (70:1) arc (70:200:1);
    \fill[gray!90!black] (0,0) circle (1cm);
  \end{scope}
  \draw [->,dotted] (-1.2,0) -- (1.2,0);
  \draw [->,dotted] (0,-1.2) -- (0,1.2);
  \draw (0,0) circle (1cm);
  \draw [->] (0,0) -- (0.6,0.6) node[below,xshift=8pt] {$x_{r}$};
  \draw [->] (0,0) -- (-0.6,0.6) node[below,xshift=-6pt] {$x_{l}$};
  \draw [->] (0,0) -- (0,-0.3) node[right] {$x_{0}$};
\end{tikzpicture}
\caption{Instances used in  multidimensional lower bound.}
\label{fig:LowerBound2Dim}
\end{figure}
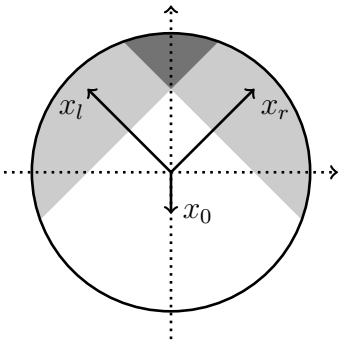

Consider distributions that choose instances among the set $\{ x_0,x_l,x_r \}$ depicted in \cref{fig:LowerBound2Dim}.  
The shaded areas in \cref{fig:LowerBound2Dim} depict regions in the domain $\W$ where either $\ell(\cdot,x_l)$ or $\ell(\cdot,x_r)$ is approximately linear. 
The dark area represents the region in which both loss functions are approximately linear.
By  setting the probability of $x_{0}$ much larger than the others we can construct a distribution over the instances $\{ x_0,x_l,x_r\}$ such that the minima of the induced expected loss function lies in the black area.
Perturbing this distribution by $\pm\eps$ over the odds of choosing $x_l,x_r$ we attain two distributions $\D_{+},\D_{-}$ whose induced expected losses $L_{+},L_{-}$ are almost linear over in the dark area, with opposed $\pm \eps$ slopes. 
An algorithm that attains a low expected excess loss on both distributions can be used to distinguish between them.
This allows us to use information theoretic arguments to lower bound the number of samples needed for the optimization algorithm.
\begin{figure}[h]
\begin{framed}
 \textbf{Input}: $\eps>0$, $D \geq 2$, $\chi\in\{-1,1\}$
\begin{enumerate}
\item Set $p \in [0,1]$ such that $\frac{p}{1-p} = \frac{D}{\sqrt{2}}\frac{1+e^{0.9}}{1+e^{-0.9D/\sqrt{2}}}$ and define:
\begin{align*}
x_0 = \tfrac{1}{D} (0,-1)\tr ~, \quad
x_l =  \tfrac{1}{\sqrt{2}}(-1,1)\tr ~, \quad 
x_r = \tfrac{1}{\sqrt{2}}(1,1)\tr
\end{align*} 
\item Define a distribution $\D_{\chi}$, that choose instances $x \in \{x_0,x_l,x_r \}$ as follows:
\begin{equation}\nonumber
x=
\begin{cases}
 x_0	&\quad \text{w.p. ~~ $p$ } \\ 
 x_l	&\quad \text{w.p. ~~ $\frac{1+\chi\eps}{2} \cdot (1-p)$ }\\ 
 x_r	&\quad \text{w.p. ~~ $\frac{1-\chi\eps}{2} \cdot (1-p)$ }
\end{cases}
\end{equation}
\end{enumerate}
 \textbf{Output}: For $t=1,2,\ldots$, sample  $x_t\sim \D_{\chi}$
\end{framed}
\caption{Two distributions: $\D_{\chi}$, $\chi\in \{-1,1\}$; any algorithm that attains an $\eps$  excess logistic loss on both of them requires $\Omega(D/\eps^{2})$ observations.}
\label{fig:Adversary2Diml}
\end{figure}
In \cref{fig:Adversary2Diml} we present the distributions $\D_{+},\D_{-}$. 
We denote by $L_{+}$ and $L_{-}$ the expected loss functions induced by $\D_{+}$ and $\D_{-}$ respectively, that are given by
\begin{align*}
	L_{\chi}(w)
	~=~ p \cdot \ell(w,x_0)+\frac{1+\chi \eps}{2} (1-p) \cdot \ell(w,x_l) + \frac{1-\chi\eps}{2} (1-p) \cdot \ell(w,x_r), 
	\qquad \chi\in\{-1,1\}
\end{align*} 
In the following lemma we state that it is impossible attain a low expected excess loss on both $\D_{+}$ and $\D_{-}$ simultaneously.
\begin{lemma}\label{Lemma:TwoAdversaries}
Given $D \ge 2$ and $\Omega(e^{-D}) \le \eps \le 1/10D$, consider $\D_{+}, \D_{-}$ as defined in \cref{fig:Adversary2Diml}. Then the following holds:
\begin{align*}
	L_{+}(w)-\min_{w^*\in\W}L_{+}(w^*) 
		&~\ge~  \eps/20 ~, 
		\qquad \forall ~ w: w[1]\leq 0 ~, \quad\text{and}  \\
	L_{-}(w)-\min_{w^*\in\W}L_{-}(w^*) 
		&~\ge~  \eps/20 ~, 
		\qquad \forall ~ w: w[1]\geq 0 ~.
\end{align*}
\end{lemma}
Here we only give a sketch of the proof; for the complete proof, refer to \cref{Section:Proof:Lemma:TwoAdversaries2Dim}.
\begin{proof}[Proof (sketch)]
Let $L_0$ be the unperturbed ($\eps=0$) version of $L_{+},L_{-}$, i.e.,
$$
	L_0(w) 
	~=~ p \ell(w, x_0)+\frac{1-p}{2}\ell(w,x_l)+ \frac{1-p}{2}\ell(w,x_r) ~.
$$ 
Note that $L_{0}$ is constructed such that its minima is attained at $w_0=(0,0.9D)$, which belongs to the shaded area in \cref{fig:LowerBound2Dim}.
Thus, in the neighborhood of this minima both $\ell(w,x_l),\ell(w,x_r)$ are approximately linear.
Using linear approximations of $\ell(w,x_l),\ell(w,x_r)$ around $w_0$, we show that the value of $L_{+}$ at $w_a = (0.3D,0.9D)$  is smaller by $\eps/20$ than 
the minimal value of $L_{0}$, hence
\begin{align}\label{equation:proofsketch:Lowerbound1}
	 \min_{w^*\in \W}L_{+}(w^*)
	 ~\le~ L_{+}(w_a)
	 ~\le~ L_{0}(w_0) - \eps/20 ~.
\end{align}
Moreover, $L_{+}$ is shown to be the sum of $L_{0}$ and a function which is positive whenever $w[1] \le 0$, thus
\begin{align}\label{equation:proofsketch:Lowerbound2}
	L_{+}(w) ~\ge~ L_0(w) ~, 
	\quad \forall ~ w ~:~ w[1] \le 0 ~.
\end{align}  
Combining \cref{equation:proofsketch:Lowerbound1,equation:proofsketch:Lowerbound2} we get
 \begin{align*}
	 L_{+}(w) -\min_{w^*\in\W}L_{+}(w^*)
	 ~\ge~ L_0(w) - \big(L_{0}(w_0) - \eps/20 \big)
	 ~\ge~ \eps/20~, 
	 & \qquad \forall ~w ~:~ w[1] \le 0 ~,
\end{align*}
where the last inequality follows from $w_0$ being the minimizer of $L_0(w)$.
A similar argument shows that for predictors $w$ such that $w[1]\geq0$, it holds that $L_{-}(w)-\min_{w^*\in \W}L_{-}(w^*) \ge \eps/20$. 
\end{proof}
For the proof of \cref{Theorem:2DimLowerBound} we require a lemma that lower-bounds the minimal number of samples needed in order to distinguish between the distributions $\D_{+},\D_{-}$ defined in \cref{fig:Adversary2Diml}. 
To this end, we use the following modified version of \cref{thm:coin}.
\begin{lemma} \label{Lemma:Coin3}
Let  $p \in (0,1/2]$. 
Consider a distribution supported on three atoms with probabilities $\{q_0,(1-q_0)(p+\chi \eps),(1-q_0)(1-p-\chi \eps)\}$, with $\chi$ being either $0$ or $1$.
Any algorithm that identifies the distribution correctly with probability at least $3/4$, needs no less than $p/16(1-q_0)\eps^2$ samples.
\end{lemma}
\cref{Lemma:Coin3} can be proved similarly to \cref{thm:coin} (see \cref{appendixA:InformationTheoretic}). 
We are now ready to prove \cref{Theorem:2DimLowerBound}. 
\begin{proof}[Proof of \cref{Theorem:2DimLowerBound}]
Consider some algorithm $\A$; we will show that if $\A$ observes $T$ samples from a distribution $\D$ which is either $\D_{+}$ or $\D_{-}$, then the expected excess loss $\tilde{\epsilon}$ that $\A$ can guarantee is lower bounded by $\Omega(\sqrt{D/T})$.

The excess loss is non negative; therefore 
if $\mathcal{A}$ guarantees an expected excess loss smaller than $\tilde{\epsilon}=\eps/80$, then by Markov's inequality 
it achieves an excess loss smaller than  $\eps/20$, w.p. $\geq 3/4$.
 Denoting by  $\wbar_T$ the predictor that $\A$ outputs after $T$ samples, then according to \cref{Lemma:TwoAdversaries}, attaining an excess loss smaller than $\eps/20$ on distribution $\D_{+}$(respectively $\D_{-}$) implies $\wbar_T[1]>0$ (respectively $\wbar_T[1]<0$).
 
Since $\A$ achieves an excess loss smaller than $\eps/20$ w.p. $\geq 3/4$ for any $\D$ among $\D_{+},\D_{-}$ we can use its output to identify the right distribution w.p. $\geq 3/4$. 
This can be done as follows:
\begin{align*}
	\text{if    }\;  \wbar_T[1]\geq 0, ~~ &\text{return ``$\D_{+}$" ~;}\\
	\text{if    }\;  \wbar_T[1] <     0, ~~ &\text{return ``$\D_{-}$" ~.}
\end{align*}
According to \cref{Lemma:Coin3}, distinguishing between these two distributions w.p.$\geq3/4$ requires that the number of observations $T$ to be upper bounded as follows:
$$
	T 
	~\ge~ \frac{0.5(1-\epsilon)}{16(1-p )(2\eps)^2}
	~\ge~ \frac{D}{256} \frac{1}{\eps^2} ~,
$$ 
We used $0.5(1-\epsilon)$ as a lower bound on the bias of distribution $\D_{-}$ conditioned that
the instance $x_0$ was not chosen; since $\epsilon\leq 1/10D$, $D\geq2$ it follows that $0.5(1-\epsilon)\geq 0.25$.
We also used $2\eps$ as the bias between the distributions  $\D_{+}$ and $\D_{-}$ conditioned that
the label $x_0$ was not chosen. Finally we used $1-p  \leq 1/D$. 
The above inequality together with $\tilde{\epsilon} = \epsilon/80$ yields a lower bound of $\frac{1}{1300}\sqrt{{D}/{T}}$ on the expected excess loss.
\end{proof}

\subsection{Lower Bounds for Online Optimization}\label{section:LowerBoundsOnline}
In \cref{section:LowerBounds} we  proved two lower bounds for the convergence rate of stochastic logistic regression. 
Standard online-to-batch conversion \citep{Cesa04} shows that any online algorithm attaining a regret of $R(T)$  can be used to attain a convergence rate of $R(T)/T$ for stochastic optimization. 
Hence, the lower bounds stated in \cref{Theorem:1DimLowerBound,Theorem:2DimLowerBound} imply the following:
\begin{corollary}\label{Corollary:1DimLowerBound}
Consider the one dimensional online logistic regression setting with  $T = \O(e^{D})$. 
For any algorithm $\A$ there exists a sequence of loss functions such that $\A$ suffers a regret of at least $\Omega(D^{2/3}T^{1/3})$.
\end{corollary}
\begin{corollary}\label{Corollary:2DimLowerBound}
Consider the multidimensional online logistic regression setting with $T = \O(e^{D})$, $D\geq 2$.
For any algorithm $\A$ there exists a sequence of loss functions such that $\A$ suffers a regret of at least $\Omega(\sqrt{DT})$.
\end{corollary}

\section{Upper Bound for One-dimensional Regression} \label{section:UpperBounds}
In this section we consider online logistic regression in one dimension; here an adversary  chooses instances $x_t \in [-1, 1]$, then
 a learner chooses predictors $w_t \in \W = \{w \in \reals : |w| \leq D\}$, and suffers a logistic loss  $\ell(w_t,x_t)=\log(1+e^{x_t w_t})$.
We provide an upper bound of $\O(T^{1/3})$ for logistic online regression in one dimension, thus showing that the lower bound found in 
\cref{Theorem:1DimLowerBound} is tight. 
Formally, we prove:
\begin{theorem} \label{Theorem:LogitFTRL}
Consider the one dimensional online regression with logistic loss. Then a player that chooses predictors $w_t\in \W$ according 
to \cref{algorithm:LogitFTRL} with $\eta=T^{-1/3}$ and $D\geq2$, achieves the  following guarantee: 
$$
\text{Regret}_T
~=~ \sum_{t=1}^T \log(1+e^{x_t w_t}) - \min_{w\in \W}\sum_{t=1}^T \log(1+e^{x_t w}) 
	~=~ \O(D^3 \, T^{1/3}) ~.
$$
\end{theorem}
Using standard online-to-batch conversion techniques \cite{Cesa04}, we can translate the upper bound given in the above lemma to an upper bound for stochastic optimization.
\begin{corollary} \label{Corollary:LogitFTRL}
Consider the one dimensional stochastic logistic regression setting with $D\geq2$ and a budget of $T$ samples. Then for any distribution $\D$ over instances, an algorithm that chooses predictors $w_1,\ldots,w_t \in \W$ according to \cref{algorithm:LogitFTRL} with $\eta=T^{-1/3}$ and outputs $\wbar_T = \frac{1}{T}\sum_{\tau=1}^T w_\tau$, achieves the following guarantee: 
$$
	\E[L(\wbar_T)]~ - \min_{w^{\st}  \in [-D,D] } L(w^{\st})
	~=~ \O(D^3 /T^{2/3}) ~.
$$
\end{corollary}
%
%
Following \cite{Zinkevich03} and \cite{HazanAK07}, we approximate the losses received by the adversary,  and use the approximate losses in a follow-the-regularized-leader (FTRL) procedure in order to choose the predictors. 
\begin{figure}[h]
\centering
\subfigure[Mixed linear/quadratic approximation]{ \label{fig:ApplossesLinquad}
\begin{tikzpicture}
    \begin{axis}[
      	  width=0.4\textwidth,
	  height=0.3\textwidth,
	  scale only axis,
	  xmin=-4,xmax=4,
	  ymin=-1.5,ymax=4,
	  axis lines=middle, 
	  ticks=none, 
	  line width=1pt]
    \addplot[domain=-4:4] {ln(1+exp(x*(1-1/10)))}
    node[pos=0.1,above] {$\ell(\cdot,x_t)$};
    \addplot[domain=0:4, densely dashed] {1.95+0.77*(x-2)};
    \addplot[domain=-4:0, densely dashed] {1.95+0.77*(x-2)+0.1*x^2}
    node[pos=0.2,above] {$\tilde{\ell}_t$};
    \draw[dotted] (axis cs:2,0) node[below] {$w_t$} -- (axis cs:2,1.95);
    \end{axis}
\end{tikzpicture}
}
\hskip 0.05\textwidth
\subfigure[Quadratic approximation]{ \label{fig:ApplossesQuad}
\begin{tikzpicture}
    \begin{axis}[
	  width=0.4\textwidth,
	  height=0.3\textwidth,
	  scale only axis,
	  xmin=-4,xmax=4,
	  ymin=-1.5,ymax=4,
	  axis lines=middle, 
	  ticks=none, 
	  line width=1pt]
    \addplot[domain=-4:4] {ln(1+exp(x*(1-1/10)))}
    node[pos=0.9,above=7pt] {$\ell(\cdot,x_t)$};
    \addplot[domain=-4:4, densely dashed] {0.34+0.26*(x+1)+0.03*(x+1)^2}
    node[pos=0.9,below] {$\tilde{\ell}_t$};
    \draw[densely dotted] 
        (axis cs:-1,0) node[below] {$w_t$} 
        -- (axis cs:-1,0.3);
    \end{axis}
\end{tikzpicture}
}
\caption{Approximate losses used by \cref{algorithm:LogitFTRL}.} \label{fig:1DimUpper}
\end{figure}
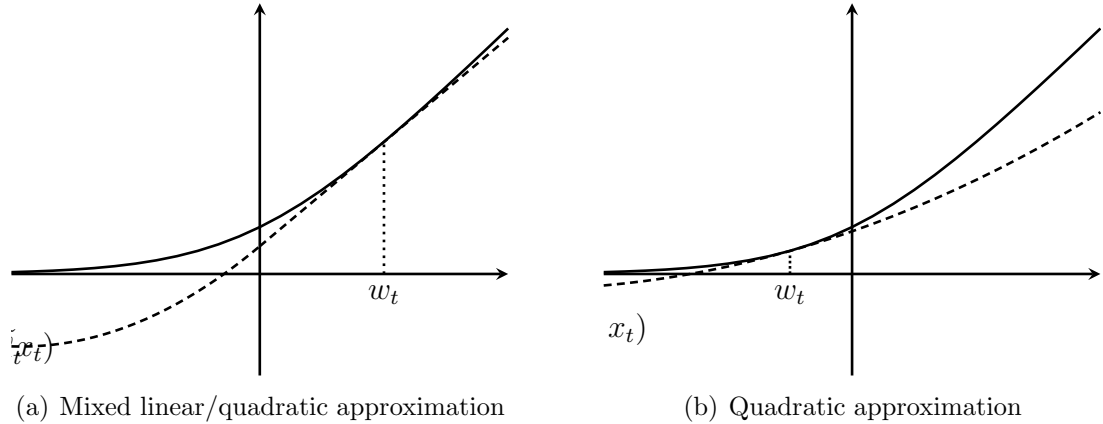

First note the following lemma due to \cite{Zinkevich03} (proof is  found in  \cite{HazanAK07}):
\begin{lemma} \label{Lemma:ApproximateLoss}
Let $\ell_1,\ldots,\ell_{T}$ be an arbitrary sequence of loss functions, and let $w_{1},\ldots,w_{T} \in\mathcal{K}$. 
Let, $\t{\ell}_1,\ldots,\t{\ell}_{T}$ be a sequence of loss function that satisfy $\t{\ell}_t(w_t) = \ell_t(w_t)$, and $\t{\ell}_t(w) \le \ell_t(w)$ for all $w\in \mathcal{K}$.
Then
$$
	\sum_{t=1}^T \ell_t(w_t) - \min_{w\in \K}\sum_{t=1}^T \ell_t(w)
	~\le~ \sum_{t=1}^T \t{\ell}_t(w_t)-\min_{w\in \K}\sum_{t=1}^T \t{\ell}_t(w) ~.
$$
\end{lemma}
Thus, the regret on the original losses is bounded by the regret of the approximate losses. For the logistic losses, $\ell(w,x_t) = \log(1+e^{x_t w})$, we define approximate losses $\t{\ell}_t$ that satisfy the conditions of the last lemma. 
Depending on $x_t,w_t $, we divide into 3 cases:
\begin{equation}\label{Equation:ApproximationLoss}
	\tilde{\ell}_t(w)=
	\begin{cases}
		a_0 + y_t w+\frac{\beta}{2}y_t^2 w^2 \ind{w \le 0}
			&\qquad \text{if $w_t\geq 0$ ~and~ $x_t \ge \frac{1}{D}$~;} \\
		a_0 + y_t w +\frac{\beta}{2} y_t^2 w^2 \ind{w \ge 0}        
			&\qquad \text{if  $w_t\leq 0$ ~and~ $x_t \le -\frac{1}{D}~;$} \\
		a_0 + y_t w+\frac{\beta}{2}y_t^2(w-w_t)^2    
			&\qquad \text{if $|x_t|\leq\frac{1}{D}$ ~or~ $x_t w_t\leq 0~,$ }
	\end{cases}
\end{equation}
where,
\begin{align*}
	y_t = \left. \frac{\partial\ell(w,x_t)}{\partial w}\right \vert_{w_t}=g_t x_t ~,\quad 
	g_t = \frac{e^{x_t w_t}}{1+e^{x_t w_t}} ~,\quad 
	\beta = 1/8D ~, \quad 
	a_0 = \log(1+e^{x_t w_t}) -g_t x_t w_t ~.
 \end{align*}
Thus, if $|x_t|\leq 1/D$ or $x_t w_t\leq 0 $, then we use a quadratic approximation, else we use a loss that changes from  linear to quadratic on $w=0$.
Note that if  the approximation loss $\tilde{\ell}_t$ is partially linear, then the magnitude of its slope $|y_t|$ is greater than $1/2D$.

The approximations are depicted in \cref{fig:1DimUpper}. In \cref{fig:ApplossesLinquad} the approximate
loss changes from linear to quadratic in $w=0$ , where in \cref{fig:ApplossesQuad} the approximate loss is quadratic everywhere.
The following technical lemma states that the losses $\{\t{\ell}_t\}$ satisfy the conditions of \cref{Lemma:ApproximateLoss}.
\begin{lemma} \label{Lemma:SatisfyApproximateLoss}
Assume that $D \ge 2$.
Let $\ell(\cdot,x_1),\ldots,\ell(\cdot,x_T)$ be a sequence of logistic loss functions and let $w_1,\ldots,w_{T} \in \W$.
The approximate losses $\t{\ell}_1,\ldots,\t{\ell}_{T}$ defined above satisfy $\t{\ell}_t(w_t) = \ell(w_t,x_t)$ and $\t{\ell}_t(w) \le \ell(w,x_t)$ for all $w\in \W$.
\end{lemma}
\cref{Lemma:SatisfyApproximateLoss} is proved in \cref{section:proof:Lemma:SatisfyApproximateLoss}.
We are now ready to describe our algorithm that 
obtains a regret of $\O(D^3T^{1/3})$ for one-dimensional online regression, given in \cref{algorithm:LogitFTRL}.
\begin{algorithm}[H]
\caption{FTRL for logistic losses}
\label{algorithm:LogitFTRL}
\begin{algorithmic}
	\STATE \textbf{Input}: Learning rate $\eta>0$,  diameter $D$ 
	\STATE let $R(w) = \frac{1}{16D}  w^2$
	\FOR{$t=1,2 \ldots T$ }
		\STATE set $w_t=\argmin_{w \in [-D,D]} \big\{ \sum_{\tau=1}^{t-1} \t{\ell}_\tau(w) + \frac{1}{\eta} R(w) \big\}$ 
		\STATE observe $x_t\in[-1,1]$ and suffer loss $\ell(w_t,x_t)=\log(1+e^{x_t w_t})$
		\STATE compute $\t{\ell}_t$ according to \cref{Equation:ApproximationLoss}
	\ENDFOR
\end{algorithmic}
\end{algorithm}
We conclude with a proof sketch of \cref{Theorem:LogitFTRL}; the complete proof is deferred to \cref{section:proof:LogitFTRL}. 
%
\begin{proof}[Proof of \cref{Theorem:LogitFTRL} (sketch)]
First we show that the regret of \cref{algorithm:LogitFTRL} is upper bounded by the sum of differences $\sum_{t=1}^T \t{\ell}'_t(w_t)(w_t-w_{t+1})$,  
and then divide the analysis into two cases. In the first case we show that the accumulated regret in rounds where  $\t{\ell}_t$ is quadratic around $w_t$  is upper bounded by  $\O(D\log{T})$. The second case analyses rounds in which $\t{\ell}_t$ is linear around $w_t$; due to the regularization, in the first such  $T^{2/3}$ rounds our regret is bounded by $\O(T^{1/3})$ and if the number of such rounds is greater than $T^{2/3}$ we show that the quadratic part of the accumulated losses is large enough so the above sum of differences is smaller than $\O(D^3T^{1/3})$.
Since the approximations $\t{\ell}_t$ may change from linear to quadratic in $w=0$,  our analysis splits into two cases:  the case where consecutive predictors $w_t,w_{t+1}$ have the same sign, and the case where they have opposite signs.
\end{proof}

\section{Proofs} \label{Section:Proofs}

\subsection{Proof of \cref{Lemma:TwoAdversaries1Dim}}\label{Section:Proof:Lemma:TwoAdversaries1Dim}

\begin{proof}
We assume that the following holds:
$$\Omega(e^{-D})=40e^{-0.45D}\leq \eps \leq \frac{1}{25}~.$$
In the proof we use the following:
$$\theta D \leq 0.2;\qquad   1-\frac{\theta}{2} \geq 0.9~,$$
the first follows since: $\theta D= \sqrt{\eps }\leq \sqrt{\frac{1}{25}}= 0.2$, combing the latter with $D\geq 1$ we get $1-\frac{\theta}{2} \geq 0.9$.
Next we prove the lemma in three steps:
\paragraph{Step 1: Linear/quadratic approximation in $\left[D/2,D\right]$.}
We show that for $w\in \left[D/2,D\right]$, the logistic losses of the instances  $(1-\frac{\theta}{2}),-\theta$ are linear/quadratic, up to an additive term of $\Delta \leq \eps/40$:
 \begin{align} 
 & \abs{ \ell(w,1-\frac{\theta}{2}) - (1-\frac{\theta}{2})w }
 =\log(1+e^{-(1-\frac{\theta}{2})w})\leq  e^{-(1-\frac{\theta}{2})w}\leq e^{-0.45D}\leq \Delta, \qquad \forall w\in [{D}/{2},D] \label{Equation:Llll}   \\ 
 &\abs{ \ell(w,-\theta) - \left(\log{2}-\frac{\theta}{2}w+\frac{(\theta w)^2}{8}\right) }
 \leq \max_{\bar{w}\in [-D,D]} \frac{(\theta \bar{w})^4}{192}\leq  \label{Equation:LinearizeLogit2Dim1}
  \frac{(\theta D)^4}{192}\leq \Delta, \qquad \forall w\in [-D,D]
   \end{align}
 recalling $\ell(w,x) = \log(1+e^{xw})$, in the first equality of \cref{Equation:Llll} we used, $\log(1+e^z)=z+\log(1+e^{-z})$, next we used $\log(1+z)\leq z$ , finally we used $w\geq D/2$ and $(1-\frac{\theta}{2})\geq 0.9$. 
In \cref{Equation:LinearizeLogit2Dim1} we used the second order taylor approximation of the loss around $0$, and the RHS of the second inequality is an upper bound to the error of this approximation. 
We define $\Delta =  \max\{e^{-0.45D},\frac{(\theta D)^4}{192} \}$; 
using $\theta ={\frac{\sqrt{\eps}}{D}}$,  $40e^{-0.45D}\leq \eps\leq \frac{1}{25}$ and $D\geq 1$ we can bound:
 $$ \Delta \leq {\eps }/{40}~.$$
 
\paragraph{Step 2: proving the lemma for  $w \in [D/2,D]$.}
Recall the notation $L_{+}(w), L_{-}(w)$ for the expected losses according to $\D_{+},\D_{-}$; using \cref{Equation:Llll,Equation:LinearizeLogit2Dim1}, we can write:
\begin{align*}
L_{+}(w)&= 
\left(\frac{\theta}{2}+\frac{\eps}{D}\right)\left(1-\frac{\theta}{2} \right)w+ \left(1-\frac{\theta}{2}-\frac{\eps}{D}\right)\left(\log{2}-\frac{\theta}{2}w+\frac{(\theta w)^2}{8}\right)\pm \Delta \\
& = \frac{\eps}{D} w + \left(1-\frac{\theta}{2}-\frac{\eps}{D}\right)\log{2} + \left(1-\frac{\theta}{2}-\frac{\eps}{D}\right)\frac{(\theta w)^2}{8}\pm \Delta, \qquad  \forall w\in \left[{D}/{2},D\right]~.
 \end{align*}

Using the latter expression for $L_{+}$ we can bound the excess loss for $w\in [D/2,D]$ as follows:
\begin{align*}
	L_{+}(w)-\min_{w^*\in \W}L_{+}(w^*)
	&\geq L_{+}(w)-L_{+}(D/2)  \\
	&\geq \frac{\eps}{D}\left(w-\frac{D}{2}\right) + \frac{\theta^2}{8}\left(1-\frac{\theta}{2}-\frac{\eps}{D}\right)\left(w^2 - \frac{D^2}{4}\right)-2\Delta \\
	&\geq \frac{\eps}{D}\left(w-\frac{D}{2}\right) + \frac{\theta^2}{10}\left(w^2 - \frac{D^2}{4}\right)-2\Delta~,
\end{align*}
where in the last inequality we used $\theta/2\leq 0.1$ and $\eps/D\leq 1/25$.
Hence, for $w\geq 3D/4$, we have
\begin{align*}
	L_{+}(w)-\min_{w^*}L_{+}(w^*) 
	~\geq~ \frac{\eps}{4} + \frac{\theta^2}{10}\frac{5D^2}{16}-2 \Delta
	~\geq~ \frac{\eps }{20} ~,  
\end{align*}
where we used $\Delta \leq {\eps }/{40}$.
 
Similarly to $L_{+}$ we can show that
\begin{align*}
	L_{-}(w) =  -\frac{\eps}{D} w + \left(1-\frac{\theta}{2}+\frac{\eps}{D}\right)\log{2} + \left(1-\frac{\theta}{2}+\frac{\eps}{D}\right)\frac{(\theta w)^2}{8}\pm \Delta ~, 
	\qquad  \forall w\in\left[{D}/{2},D\right] ~.
\end{align*}
 Using the latter expression for $L_{-}$ we can bound the excess loss for $w\in [D/2,D]$ as follows:
\begin{align*}
	L_{-}(w)-\min_{w^*\in \W}L_{-}(w^*)
	&\geq L_{-}(w)-L_{-}(D)  \\
	&\geq -\frac{\eps}{D}(w-D) + \frac{\theta^2}{8}\left(1-\frac{\theta}{2}-\frac{\eps}{D}\right)(w^2 - D^2)-2\Delta \\
	&\geq -\frac{\eps}{D}(w-D) + \frac{\theta^2}{8}\left(\frac{D^2}{4} - D^2\right)-2\Delta ~.
\end{align*}
Hence, for $w\in [\frac{D}{2},\frac{3D}{4}]$, we have:
\begin{align} \label{Equation:PartialArea1Dim}
	L_{-}(w)-\min_{w^*}L_{-}(w^*) 
	 ~\geq~ \frac{\eps}{4} - \frac{\theta^2}{8}\frac{3D^2}{4}-2 \Delta 
	 ~\geq~ \frac{\eps }{20} ~.  
\end{align}
%

\paragraph{Step 3: Extending the lemma to  $w \in [-D,D]$.} 

We are left to prove:
$$L_{-}(w)-\min_{w^*}L_{-}(w^*) \geq \frac{\eps }{20},  \qquad  \forall w\in [-D,{D}/{2}]~.$$
According to \cref{Equation:PartialArea1Dim}, 
 it suffices to prove  $L_{-}(w)\geq L_{-}({D}/{2}),\; \forall w\in [-D,{D}/{2}]$. Since $L_{-}$ is convex, showing that the derivative of $L_{-}$ at $D/2$ is negative implies that 
  $L_{-}(w) \geq L_{-}({D}/{2}),\; \forall w\leq {D}/{2}$. Deriving $L_{-}(w)$ at ${D}/{2}$ we get:
 \begin{align*}
\left. \frac{d}{dw} L_{-}(w) \right \vert_{D/2}
	&= \left(\frac{\theta}{2}-\frac{\eps}{D}\right)\left(1-\frac{\theta}{2}\right)\frac{1}{1+e^{-\left(1-\frac{\theta}{2}\right)\frac{D}{2}}} - \theta\left(1-\frac{\theta}{2}+\frac{\eps}{D}\right)\frac{1}{1+e^{\frac{\theta D}{2}}} \\
 &\leq \left(\frac{\theta}{2}-\frac{\eps}{D}\right)\left(1-\frac{\theta}{2}\right)-
  \theta\left(1-\frac{\theta}{2}+\frac{\eps}{D}\right)\left(\frac{1}{2}-\frac{\theta D}{8}\right) \\
  &\leq -\frac{\eps}{D}+\frac{\theta^2 D}{8}
  = -\frac{\eps}{D} +\frac{\eps}{8D} \leq 0~,
 \end{align*}   
where in the first inequality we used $(1+e^x)^{-1}\leq 1,\forall x$, and $(1+e^x)^{-1}\geq \frac{1}{2}-\frac{x}{4},\forall x\geq 0$, this is since $(1+e^x)^{-1}$ is convex for $x\geq 0$
and $\frac{1}{2}-\frac{x}{4}$ is its tangent at $x=0$. In the last line we used $\theta= {\sqrt{\eps}/D}$.
 \end{proof}

\subsection{Proof of \cref{Lemma:TwoAdversaries}}\label{Section:Proof:Lemma:TwoAdversaries2Dim}

\begin{proof}
We assume that the following holds:
 $$\Omega(e^{-D})=100e^{-0.6D/\sqrt{2}}\leq \epsilon\leq \frac{1}{10D} = \O(1/D)~.$$
In the proof we will need to use: $\frac{1}{6D}\leq \frac{1-p}{2}\leq \frac{1}{2D}$, this can be shown by simple algebra using the definition of $p$ in \cref{fig:Adversary2Diml} and using  $D\geq 2$. 
Next we prove the lemma in three steps:
 \paragraph{Step 1: Define $L_{0}(w)$ and find its minima.} Define:
 $$L_0(w) = p \ell(w, x_0)+\frac{1-p}{2}\ell(w,x_l)+ \frac{1-p}{2}\ell(w,x_r) ~,$$ 
 where $p$ is defined in \cref{fig:Adversary2Diml}. Note that $L_{0}$ is the unperturbed version ($\eps=0$) of $L_{+},L_{-}$. 
 We want to show that $w_0=(0, 0.9D)$ is the global minimizer of $L_{0}$; since $L_{0}(w)$ is convex it is sufficient to show that $\nabla L_{0}(w_0)=0$.
 Deriving $L_0$ we get
$$
	\nabla L_0(w) 
	= \frac{p}{D}  \frac{1}{1+e^{ w[2]/D}}(0,-1) \\
	+\frac{1-p}{2\sqrt{2}}\left( 
		\frac{1}{1+e^{(w[1]-w[2])/\sqrt{2}}}(-1,1) 
		+\frac{1}{1+e^{-(w[1]+w[2])/\sqrt{2}}}(1,1) 
	\right) ~.
$$
 substituting  $p$ so that $\frac{p}{1-p} = \frac{D}{\sqrt{2}}\frac{1+e^{0.9}}{1+e^{-0.9D/\sqrt{2}}}$, and $w_0=(0,0.9D)$ confirms that the gradient is indeed 
 zero at $w_0$.

\paragraph{Step 2: Bounding the minimal loss of $L_{+}(w)$.}
We would  like to upper bound the minimal value of $L_{+}(w)$ as follows:
 $$\min_{w^*\in \W}L_{+}(w^*)\leq L_{0}(w_0)-\eps/20~.$$ 
We do so by showing that for $w_a = (0.3D,0.9D)$ it holds that $L_{+}(w_a)\leq L_{0}(w_0)-\eps/20$.

First, notice that we can write $w_a = w_0 + u_a$, where $u_a=(0.3D,0)$. Recalling $\ell(w,x) = \log(1+e^{x\cdot w})$, we use $x_0 \cdot w_a = x_0 \cdot w_0 = -0.9$ to  get:
 \begin{align}\label{Equation:LinearizeLogit1}
 &\ell(w_a, x_0)=\ell(w_0, x_0)~.
 \end{align}
 Moreover:
 \begin{align}\label{Equation:LinearizeLogit2}
 \ell(w_a,x_l)&= x_l \cdot w_a + \log(1+e^{-x_l \cdot w_a}) = x_l \cdot w_0 + x_l \cdot u_a+\log(1+e^{ -\frac{0.6D}{\sqrt{2}}}) \\ \nonumber
 &\leq \ell(w_0,x_l) -\frac{0.3D}{\sqrt{2}}+e^{ -\frac{0.6D}{\sqrt{2}}}~,
 \end{align}
 recalling $\ell(w,x) = \log(1+e^{x\cdot w})$, in the equalities we used $\log(1+e^z) = z+\log(1+e^{-z})$, and $x_l \cdot w_a=\frac{0.6D}{\sqrt{2}}$; 
 In the inequality we used $x_l \cdot u_a =-\frac{0.3D}{\sqrt{2}}$, next we used  $z\leq \log(1+e^z)$, and also $\log(1+z)\leq z$. Similarly to \cref{Equation:LinearizeLogit2}, we can show:
 \begin{align}\label{Equation:LinearizeLogit3}
  \ell(w_a,x_r) \leq \ell(w_0,x_r)+ \frac{0.3D}{\sqrt{2}}+e^{ -\frac{1.2D}{\sqrt{2}}}~.
 \end{align}
 Now, plugging \cref{Equation:LinearizeLogit1,Equation:LinearizeLogit2,Equation:LinearizeLogit3} into the definition of  $L_{+}(w_a)$, we get:
\begin{align}\label{Equation:PertubizedValueEstimate}
L_{+}(w_a)
&\leq p \ell(w_0,x_0)  + \frac{1-p }{2}(1+\eps)\left(\ell(w_0,x_l) -\frac{0.3D}{\sqrt{2}}+e^{ -\frac{0.6D}{\sqrt{2}}}\right) \nonumber\\ 
&\qquad+ \frac{1-p }{2}(1-\eps)\left(\ell(w_0,x_r)+ \frac{0.3D}{\sqrt{2}}+e^{ -\frac{1.2D}{\sqrt{2}}}\right) \nonumber\\
&\leq L_0(w_0) -\frac{1-p }{2}(0.3\sqrt{2}D)\eps+e^{ -\frac{0.6D}{\sqrt{2}}} \nonumber\\
&\leq L_0(w_0) -\frac{\sqrt{2}}{20}\eps+e^{ -\frac{0.6D}{\sqrt{2}}} \nonumber\\
&<  L_0(w_0) -\eps/20  ~, 
 \end{align}
we used $\ell(w_0,x_l)=\ell(w_0,x_r)$, and  $\frac{1-p }{2}\geq \frac{1}{6D}$, we also used  $\eps\geq 100e^{-0.6D/\sqrt{2}}$.
So we showed that $L_{+}(w_a)$ is upper bounded by $L_0(w_0) -\eps/20$, 
thus upper bounding the minimum of $L_{+}(w)$.

\paragraph{Step 3: Bound the excess loss  of predictors $w: w[1]\leq 0$.}

In order to so, it is sufficient to show that the value of such predictors is greater by $\eps/20$ than the upper bound we found for $\min_{w^*\in \W}L_{+}(w^*)$. 
Let us write $L_{+}(w)$ as a sum of $L_0(w)$ and a perturbation:
\begin{align*}
	L_{+}(w) 
	&= L_0(w) + \frac{(1-p )\eps}{2}\big(\log(1+e^{x_l \cdot w})-\log(1+e^{x_r \cdot w})\big)\\
	&= L_0(w) + \frac{(1-p )\eps}{2}\left(\log\big(1+e^{\frac{1}{\sqrt{2}}(w[2]-w[1])}\big)-\log\big(1+e^{\frac{1}{\sqrt{2}}(w[2]+w[1])}\big)\right) \\
	&\geq L_0(w) ~, 
	\qquad \forall w[1]\leq 0 ~.
 \end{align*} 
The inequality follows since $w[1]\leq 0$ and $\log(1+e^z)$ is monotonically increasing, therefore the perturbation  summand is positive.
 Combining the  above inequality with the upper bound found in step 2 above we get:
 \begin{align*}
 L_{+}(w) -\min_{w^*\in\W}L_{+}(w^*)&\geq L_0(w) -\big(  L_{0}(w_0)-\eps/20 \big) 
 \geq   \eps/20 ~, 
 \qquad \forall w:w[1]\leq0 ~,
  \end{align*}
  and the last inequality follows from $w_0$ being the minimizer of $L_0(w)$.
 We can similarly show that for predictors $w$ such that  $w[1]\geq0$, then $L_{-}(w)-\min_{w^*\in \W}L_{-}(w^*) \geq \eps/20$ applies.
\end{proof}

\subsection{Proof of \cref{Theorem:LogitFTRL}}\label{section:proof:LogitFTRL}

Since the approximate losses $\tilde{\ell}_t$ defined in \cref{Equation:ApproximationLoss} satisfy the conditions of 
\cref{Lemma:ApproximateLoss} then it suffices to prove the lower bound for the regret of the $\tilde{\ell}_t$'s.

Denoting, $F_t(w) = \sum_{\tau=1}^{t-1}\tilde{\ell}_\tau(w)+ R(w)$, then \cref{algorithm:LogitFTRL} chooses $w_t = \argmin_{w\in\W}F_t(w)$.
Letting  $u_t$ be the global minimizer of $F_t$, the following is equivalent to \cref{algorithm:LogitFTRL}:
\begin{algorithm}[H]
\caption{Equivalent form-FTRL}
\label{algorithm:Equiv1LogitFTRL}
\begin{algorithmic}
\STATE Calculate: $u_t=\argmin_{w\in \reals}\sum_{\tau=1}^{t-1} \tilde{\ell}_\tau(w) + \eta^{-1} R(w)$ 
\STATE Choose: $w_t =\argmin_{w\in\W}|w-u_t|$
\end{algorithmic}
\end{algorithm} 
\cref{algorithm:Equiv1LogitFTRL} first finds $u_t$, the global minima of $F_t$, and then projects $u_t$ onto $\W$.
 The expression for $u_t$ in \cref{algorithm:Equiv1LogitFTRL} is useful since it enables us to calculate the differences  $|u_{t-1}-u_{t}|$, which upper bound the differences 
 between predictors: $|w_{t-1}-w_t|$; these differences are useful in bounding the regret of FTRL as seen in the next lemma due to \cite{KalaiVempala} (proof 
 can be found in \cite{Hazan09}  or in \cite{shalev2011online}):
\begin{lemma} \label{Lemma:FTL-BTL}
Let a regularizer function $R$, and $f_t$, for $t=1,\ldots,T$, be a sequence of cost functions and let $w_t = \argmin_{w\in\mathcal{K}}\sum_{\tau=1}^{t-1}f_{\tau}(w)+\eta^{-1}R(w)$, 
Then:
\begin{align*}
\sum_{t=1}^T f_t(w_t)-\sum_{t=1}^T f_t(v)\leq \sum_{t=1}^T \nabla f_t(w_t) \cdot (w_t-w_{t+1})+\eta^{-1}(R(v)-R(w_1)), \quad \forall v\in\mathcal{K}
\end{align*}
\end{lemma}

Note that in our one-dimensional case the gradient $\nabla \tilde{\ell}_t(w)$ is simply the derivative $\tilde{\ell}'_t(w)$. Also note that we can bound the \emph{FTL-BTL differences:} 
$\tilde{\ell}'_t(w_t)(w_t-w_{t+1})$ as follows:
\begin{align}\label{equation: u_tBound}
 \tilde{\ell}'_t(w_t)(w_t-w_{t+1})\leq |\tilde{\ell}'_t(w_t)(w_t-w_{t+1})|\leq  |y_t||w_t-w_{t+1}|\leq  |y_t||u_t-u_{t+1}|~,
\end{align}
where we used $\tilde{\ell}_t'(w_t)=y_t$ (see \cref{Equation:ApproximationLoss}), we also used $|w_{t+1}-w_{t}|\leq |u_{t+1}-u_t|$ which follows from $w_t$ being the projection of
$u_t$ onto $\W=[-D,D]$.

Combining \cref{Lemma:FTL-BTL} with \cref{equation: u_tBound}, the regret of \cref{algorithm:LogitFTRL} is bounded as follows:
\begin{align}
\text{Regret}_T\leq \sum_{t=1}^T |y_t(w_{t}-w_{t+1})|+\eta^{-1}\frac{D}{16}     \leq \sum_{t=1}^T |y_t(u_{t}-u_{t+1})|+\frac{D}{16}  T^{1/3}  ~,
\end{align}
where we used $R(w) = \frac{1}{16 D} w^2 \leq \frac{D}{16}  ,\; \forall w\in \W$, and $\eta=T^{-1/3}$.
In \cref{section:sameSigns,section:oppositeSigns}, we analyze the differences $|y_t(w_{t}-w_{t+1})|$, we divide the analysis into two cases: 
\begin{enumerate}
\item rounds in which $u_t u_{t+1}\geq 0$: \cref{section:sameSigns}.
\item rounds in which $u_t u_{t+1}< 0$: \cref{section:oppositeSigns}.
\end{enumerate}

\subsubsection{Rounds in which $u_t u_{t+1}\geq 0$}\label{section:sameSigns}
Assume without loss of generality that $u_t,u_{t+1}$, are both positive. Hence, $u_i = \argmin_{w\geq 0}F_i(w),\; i\in\{t,t+1\}$. 
For $w\geq 0$,  the losses are either linear with a  \emph{positive slope} $\geq 1/2D$, or quadratic losses, this can be seen easily from \cref{Equation:ApproximationLoss}. 
Lets introduce some notation:
\begin{align*}
y_t^q = y_t\mathbbm{1}_{\{\text{$|x_t|\leq 1/D$ or $x_t w_t\leq0$ }\}}+ y_t \mathbbm{1}_{\{ \text{$w_t\leq 0$, $x_t\leq -1/D$ }\}}; \quad y_t^{l} = y_t \mathbbm{1}_{\{  \text{ $w_t\geq 0$, $x_t\geq 1/D$}\}}; 
\end{align*}
The notation``$q$", stands for quadratic losses on $w\geq0$,  the ``$l$" notation is for losses that are linear on $w\geq 0$. We will also use the following notation $\hat{w}_t$:
\begin{equation}
\hat{w}_t=
\left\{
\begin{aligned}
0;        &\quad \text{if  $w_t\leq 0$, $x_t\leq -\frac{1}{D}$ }\\
w_t;    &\quad \text{otherwise }
\end{aligned} 
\right.
\end{equation}
Using these new notations, and the expression for the $\tilde{\ell}_t$'s in \cref{Equation:ApproximationLoss}, then $\forall w\geq0$:
\begin{align*}
F_{t}(w) &= \sum_{\tau=1}^{t-1} \tilde{\ell}_\tau(w) + \eta^{-1}\frac{\beta}{2} w^2 
= \sum_{\tau=1}^{t-1}y_{\tau}w +
 \frac{\beta}{2}\sum_{\tau = 1}^{t-1}(y_\tau^q)^2(w-\hat{w}_\tau)^2 + \eta^{-1}\frac{\beta}{2}w^2   ~,
\end{align*}
where we used $R(w)=\frac{1}{16D}w^2=\frac{\beta}{2}w^2$.
From the last expression we can derive an analytic expression for $u_t$:
\begin{align}\label{Equation:GlobalMinimaPositive}
u_t =\argmin_{w\geq 0}F_t(w) =-\frac{1}{\beta}\frac{\sum_{\tau=1}^{t-1}y_\tau - \beta \sum_{\tau}^{t-1} (y_\tau^q)^2 \hat{w}_\tau}{\sum_{\tau}^{t-1}(y_{\tau}^{q})^2+\eta^{-1}} ~. 
\end{align}
Next we analyze the sum of differences $\sum_{\tau=1}^t y_\tau(w_{\tau}-w_{\tau+1})$, the analysis divides into two sub-cases, first we analyze rounds in which 
 $\tilde{\ell}_t$ is quadratic, and then we analyze rounds where $\tilde{\ell}_t$ is linear:

\paragraph{Rounds when $\tilde{\ell}_t$ is quadratic for $w\geq0$:}
 In that case $y_t=y_t^q$ and we have: 
\begin{align*}
\tilde{\ell}_t(w) &= y_t^q w+\frac{\beta}{2}(y_t^q)^2(w-{w}_t)^2, \qquad \forall w\geq 0  ~,
\end{align*}
for such a quadratic loss $\tilde{\ell}_t$, then \cref{Equation:GlobalMinimaPositive} provides an analytic expression for $u_{t+1}$, subtracting $u_t$ is can be shown that:
\begin{align}\label{u_tLinearDiff}
u_{t} -u_{t+1} =\frac{1}{\beta}\frac{y_t^q+\beta (y_t^q)^2(u_t-w_t)}{\sum_{\tau=1}^{t-1}(y_\tau^{q})^2+\eta^{-1}}~,
\end{align}
If both $u_t,u_{t+1}\geq D$, it means that $w_t=w_{t+1}=D$, and therefore:
\begin{align}
y_t^q(w_t-w_{t+1})=0~.
\end{align}
If either $u_t< D$ or $u_{t+1}< D$, we have:
\begin{align}
|y_t^q(w_t-w_{t+1})| &\leq   |y_t^{q}(u_{t}-u_{t+1})| =\frac{1}{\beta}\frac{(y_{t}^{q})^2}{\sum_{\tau=1}^{t}(y_\tau^{q})^2 +\eta^{-1}}|1+\beta y_t^{q}(u_t-w_t)| \nonumber \\
&\leq \frac{2}{\beta}\frac{(y_{t}^{q})^2}{\sum_{\tau=1}^{t}(y_\tau^{q})^2 +\eta^{-1}}~,
\end{align}
 where we used the inequality $|u_t-u_{t+1}|\leq 4D$ (can be derived from the expressions for $u_t,u_{t+1}$ ), and thus if either $u_t,u_{t+1}$ is smaller than $D$ it follows $|u_t|\leq 5D$, we then use  $|w_t|\leq D, |u_t|\leq 5D$, $\beta=\frac{1}{8D}$, and $|y_t^{q}|\leq 1$  to show that
 $|1+\beta y_t^{q}(u_t-w_t)|\leq 2$.
 
 Thus,  for rounds in which $u_t u_{t+1}\geq 0$ and $\tilde{\ell}_t$ is quadratic, we can bound the regret by:
 \begin{align}\label{equation:upperboundQuadLoss}
  \frac{2}{\beta}\sum_{t=1}^T \frac{(y_{t}^{q})^2}{\sum_{\tau=1}^{t}(y_\tau^{q})^2 +\eta^{-1}}\leq 16D\log(T+1)~,
  \end{align}
 where we used $\beta=1/8D$ together with the following lemma, taken from  \cite{HazanAK07}:
\begin{lemma}
Let $v_t \in \reals$, for $t=1,\ldots,T$, be a sequence of scalars such that for some $r$, $|v_t|\leq r$. Then:
\begin{align*}
\sum_{t=1}^T \frac{v_t^2}{\sum_{\tau=1}^t v_\tau^2+\eps}\leq \log(r^2T/\eps+1)~.
\end{align*}
\end{lemma} 
 \paragraph{Rounds when $\tilde{\ell}_t$ is linear for $w\geq0$:}
In that case $x_t\geq \frac{1}{D}$, $y_t = y_t^l$ and we have:
$$\tilde{\ell}_t(w) = y_t^l w,\qquad \forall w\geq 0 ~,$$
for such a linear loss $\tilde{\ell}_t$, then \cref{Equation:GlobalMinimaPositive} provides an analytic expression for $u_{t+1}$, subtracting $u_t$ is can be shown that:
\begin{align}\label{u_tLinearDiff}
u_{t} -u_{t+1} =\frac{1}{\beta}\frac{y_t^l}{\sum_{\tau=1}^{t-1}(y_\tau^{q})^2+\eta^{-1}}~.
\end{align}
We are left to bound the sum of differences $y_t(w_{t}-w_{t+1})$ at times in which $\tilde{\ell}_t$ is linear (and therefore $y_t=y_t^l$, and $x_t\geq \frac{1}{D}$); according to \cref{u_tLinearDiff} each such difference is bounded by:
\begin{align}\label{equation:BoundPosx}
y_t^l(w_{t}-w_{t+1}) = \frac{1}{\beta}\frac{(y_t^{l})^2}{\sum_{\tau=1}^{t-1}(y_\tau^{q})^2+\eta^{-1}}~.
\end{align}
Define $n_{+}(t)$, to be the number of positive linear losses received at the first $t$ rounds:
$$n_{+}(t) =\sum_{\tau=1}^t \mathbbm{1}_{\{x_\tau \geq \frac{1}{D}, w_t\geq 0\}}~.$$
Suppose that $n_{+}(T)\leq T^{2/3}$, and recall $\eta=T^{1/3}$ and $|y_t^l |\leq 1$, then:
\begin{align} \label{equation:upperboundLinLoss2}
\sum_{t=1}^T y_t^l(w_{t}-w_{t+1}) \mathbbm{1}_{\{x_t\geq \frac{1}{D}, w_t\geq 0\}}& 
=\frac{1}{\beta} \sum_{t=1}^T \frac{(y_t^{l})^2}{\sum_{\tau=1}^{t-1}(y_\tau^{q})^2+\eta^{-1} }\mathbbm{1}_{\{x_t\geq \frac{1}{D}, w_t\geq 0\}}
\leq \frac{1}{\beta}  \eta n_{+}(T) \nonumber \\
& \leq \frac{1}{\beta}T^{-1/3}T^{2/3}  = 8D T^{1/3}~.
\end{align}
Suppose on the contrary that $n_{+}(T) \geq T^{2/3}$, so till time $t_0$ for which $n_{+}(t_0)=T^{2/3}$, we accumulate a regret bounded by $8DT^{1/3}$. Next, we analyze the
FTL-BTL differences at rounds in which $x_t\geq \frac{1}{D}$, $w_t\geq 0$, and $t\geq t_0$. From \cref{Equation:GlobalMinimaPositive} for $u_t$, it can be  
seen that $u_t\geq 0 $ implies:
 $$ \sum_{\tau=1}^{t-1}y_\tau - \beta \sum_{\tau=1}^{t-1} (y_\tau^q)^2 \hat{w}_\tau \leq 0~.$$
the latter equation can be written as follows:
\begin{align}\label{Equation:GetTthird}
 \sum_{\tau=1}^{t-1}(-y_\tau)\mathbbm{1}_{\{y_\tau< 0 \}} + \beta \sum_{\tau=1}^{t-1} (y_\tau^q)^2 \hat{w}_\tau \geq \sum_{\tau=1}^{t-1} y_{\tau}\mathbbm{1}_{\{y_\tau\geq 0\}}~.
\end{align}
The RHS of the last equation can be lower bounded as follows:
\begin{align}\label{Equation:GetTthirdsecond}
\sum_{\tau=1}^{t-1} y_{\tau}\mathbbm{1}_{\{y_\tau\geq 0\}}\geq \sum_{\tau=1}^{t-1} y_{\tau}^l \geq \frac{1}{2D}n_{+}(t-1) ~,
\end{align}
where we used the definition of $n_{+}(t)$, and $y_\tau^l\geq 1/2D$. The LHS of \cref{Equation:GetTthird} is upper bounded as follows:
\begin{align}\label{equation:GettthirdLast}
\sum_{\tau=1}^{t-1}(-y_\tau)\mathbbm{1}_{\{y_\tau< 0 \}} + \beta \sum_{\tau=1}^{t-1} (y_\tau^q)^2 \hat{w}_\tau &\leq 
 \sum_{\tau=1}^{t-1}|y_\tau^q|+\frac{1}{8} \sum_{\tau=1}^{t-1}|y_\tau^q| \leq \frac{9}{8}\sqrt{(t-1)\sum_{\tau=1}^{t-1}(y_{\tau}^q)^2}~,
\end{align}
in the first inequality we used  $\sum_{\tau=1}^{t-1}(-y_\tau)\mathbbm{1}_{\{y_\tau< 0 \}}\leq \sum_{\tau=1}^{t-1}|y_\tau^q|$,  also $|\hat{w}_\tau|\leq D$, $\beta =1/8D$, and
finally $ (y_\tau^q)^2\leq |y_\tau^q|$ (since $|y_\tau^q|\leq 1$); in the second inequality we used  $||z||_1\leq \sqrt{N ||z||_2^2},\; \forall z\in \reals^N $.
Combining \cref{Equation:GetTthird,Equation:GetTthirdsecond,equation:GettthirdLast} we get:
\begin{align}\label{Equation:GetTthird3}
\sum_{\tau=1}^{t-1}(y_{\tau}^q)^2\geq  \frac{1}{10D^2}\frac{n_{+}^2(t-1)}{t-1} \geq  \frac{1}{10D^2}\frac{n_{+}^2(t-1)}{T}~.
 \end{align}
Using the inequality in \cref{Equation:GetTthird3} inside \cref{equation:BoundPosx}, then the sum of differences $y_t(w_{t}-w_{t+1})$ 
for the rounds with a linear loss (hence $y_t=y_t^l$) and $t> t_0$, we can upper bound:
\begin{align} \label{equation:upperboundLinLoss}
\sum_{t=t_0+1}^T y_t^l(w_{t}-w_{t+1})& \leq  
\frac{1}{\beta}\sum_{t=t_0+1}^T\frac{(y_t^{l})^2}{\sum_{\tau=1}^{t-1}(y_\tau^{q})^2+\eta^{-1}} 
\leq 80 D^3 T\sum_{t=t_0+1}^T \frac{1}{n_{+}^2(t-1)}\mathbbm{1}_{\{x_t\geq \frac{1}{D},w_t\geq0\}}  \nonumber \\
&\leq  80 D^3 T \sum_{i=T^{2/3}}^{n_{+}(T)} \frac{1}{i^2} \leq    80 D^3 T \frac{2}{T^{2/3}} = 160 D^3 T^{1/3}~,
\end{align}
where we assumed $n_{+}(t_0)=T^{2/3}\leq n_{+}(T)$, and used  $\beta=1/8D$,  $(y_t^{l})^2\leq 1$, finally we applied:
$$\sum_{i=n_1}^{n_2}\frac{1}{i^2} \leq \frac{1}{n_1^2}+\int_{y=n_1}^{\infty}\frac{1}{y^2}dy=\frac{1}{n_1^2}+\frac{1}{n_1}\leq \frac{2}{n_1}~.$$
Hence during rounds where $u_t u_{t+1}\geq 0$, then  \cref{equation:upperboundQuadLoss,equation:upperboundLinLoss2,equation:upperboundLinLoss} upper bound the regret of \cref{algorithm:LogitFTRL} by:
$$16D\log(T+1)+ 8D T^{1/3}+ 160 D^3 T^{1/3}  ~. $$

\subsubsection{Rounds in which $u_t u_{t+1}< 0$} \label{section:oppositeSigns}

Assume without loss of generality that,  $u_t\geq 0$, and $u_{t+1}<0$, thus, $u_t = \argmin_{w\geq 0}F_t(w)$ and
$u_{t+1} = \argmin_{w\leq 0}F_{t+1}(w)$. Since $u_t\geq0$, then according to \cref{Equation:GlobalMinimaPositive} we have:
\begin{align*}
\sum_{\tau=1}^{t-1}y_\tau - \beta \sum_{\tau}^{t-1} (y_\tau^q)^2 \hat{w}_\tau\leq 0~.
\end{align*}
Since $u_{t+1}\leq 0$, we must have: 
$$\sum_{\tau=1}^{t-1}y_\tau - \beta \sum_{\tau=1}^{t-1} (y_\tau^q)^2 \hat{w}_\tau+y_t-\beta (y_t^q)^2 \hat{w}_t \geq 0~,$$
 or else the global minima would be positive. The last two inequalities imply that:
\begin{align}
|\sum_{\tau=1}^{t-1}y_\tau - \beta \sum_{\tau=1}^{t-1} (y_\tau^q)^2 \hat{w}_\tau|\leq y_t-\beta (y_t^q)^2 \hat{w}_t \leq y_t~.
\end{align}
Combining the last equation with \cref{Equation:GlobalMinimaPositive}, we get:
\begin{align*}
u_t \leq\frac{1}{\beta}\frac{y_t}{\sum_{\tau}^{t-1}(y_{\tau}^{q})^2+\eta^{-1}} ~,
\end{align*}
and therefore:
\begin{align*}
 y_t u_t \leq \frac{1}{\beta}\frac{y_t^2}{\sum_{\tau}^{t-1}(y_{\tau}^{q})^2+\eta^{-1}} ~.
\end{align*}
Similar to the analysis made in \cref{section:sameSigns} we can show that: 
\begin{align*}
 \sum_{t=1}^T y_t u_t\mathbbm{1}_{\{u_t u_{t+1}<0\}} \leq 16D\log(T+1)+ 8D T^{1/3}+ 160 D^3 T^{1/3} ~.
\end{align*}
symmetrically, we can show:
\begin{align*}
\sum_{t=1}^T y_t u_{t+1} \mathbbm{1}_{\{u_t u_{t+1}<0\}} \geq -16D\log(T+1)- 8D T^{1/3}- 160 D^3 T^{1/3} ~.
\end{align*}
From the last two inequalities, it follows:
\begin{align*}
\sum_{t=1}^T y_t (u_t-u_{t+1})\mathbbm{1}_{\{u_t u_{t+1}<0\}} \leq 32D\log(T+1)+ 16D T^{1/3}+ 320 D^3 T^{1/3} ~.
\end{align*}

\subsubsection{Concluding the proof}

According to \cref{section:sameSigns,section:oppositeSigns},  the regret of \cref{algorithm:LogitFTRL} is upper bounded by:
$$\text{Regret}_T\leq 48D\log(T+1)+ 24D T^{1/3}+ 480 D^3 T^{1/3}+\frac{D}{16}  T^{1/3}~,   $$ 
where the last term is due to the regularization. \qed

\subsection{Proof of \cref{Lemma:SatisfyApproximateLoss} } \label{section:proof:Lemma:SatisfyApproximateLoss}

\begin{proof}
For ease of notation we use the following shorthand for the logistic loss: 
$$\ell_t(w):=\ell(w,x_t)=  \log(1+e^{x_t w})$$
The proof is divided into 4 cases:
\paragraph{Case 0.}
Denote by $\tilde{\ell}^{(0)}_t$, an approximate loss of the logistic around $w=0$, thus:
\begin{align}\label{Equation:LosstildeZero}
\tilde{\ell}_t^{(0)}(w)= \ell_t(0)+\ell_t'(0) w+\frac{\beta}{2} x_t^2 w^2 = \log(2)+\frac{x_t}{2}w + \frac{\beta}{2} x_t^2 w^2
\end{align}
where we used $\ell_t(w) = \log(1+e^{x_t w})$. Next, we show that $\tilde{\ell}_t^{(0)}(w)\leq \ell_t(w),\; \forall w\in[-D,D]$. Lets write
$\ell_t(w)-\tilde{\ell}_t^{(0)}(w)$, explicitly:
\begin{align*}
\ell_t(w)-\tilde{\ell}_t^{(0)}(w)&=\log(1+e^{x_t w})-\log(2)-\frac{x_t w}{2}-\frac{\beta}{2}(x_t w)^2=\log(1+e^{z})-\log(2)-\frac{z}{2}-\frac{\beta}{2}z^2 \\
&=\log(\frac{e^{-z/2}+e^{z/2}}{2})-\frac{\beta}{2}z^2
\end{align*}  
and  we denoted $z=x_t w$. Thus, it is sufficient to show that $\log(\frac{e^{-z/2}+e^{z/2}}{2})-\frac{\beta}{2}z^2 \geq 0,\; \forall z\in[-D,D]$.
Assume $z\in[-10,10]$, then  from the taylor expansion of $\log(\frac{e^{-z/2}+e^{z/2}}{2})$ around zero, there exists $\bar{z}: |\bar{z}|\leq 10$ such that:
\begin{align*}
\log(\frac{e^{-z/2}+e^{z/2}}{2})-\frac{\beta}{2}z^2&=\frac{z^2}{8}-\frac{\bar{z}^4}{192}-\frac{\beta}{2}z^2 \geq (\frac{1}{8}-\frac{1}{16D})z^2-\frac{z^4}{192} \\
& \geq \frac{z^2}{16}-\frac{z^4}{192}\geq 0, \qquad \forall z\in[-10,10]
\end{align*}  
where we used $\beta=\frac{1}{8D}$, $D\geq 2$, and $|\bar{z}|\leq |z|\leq 10$.
Assuming $10\leq |z|\leq D$:
\begin{align*}
\log(\frac{e^{-z/2}+e^{z/2}}{2})-\frac{\beta}{2}z^2&\geq   \log(e^{|z|/2})-\frac{\beta}{2}z^2-\log(2) = \frac{|z|}{2}-\frac{1}{16D}z^2-\log(2) \\
&\geq \frac{|z|}{2}-\frac{|z|}{8}-\log(2)\geq 0, \qquad \forall 10\leq |z|\leq D
\end{align*}  
we used $\beta=\frac{1}{8D}$, in the second inequality we used $|z|\leq D$, and in the last inequality we used $|z|\geq10$. So we have shown:
\begin{align} \label{Equation:LossZerotildeInequality}
\tilde{\ell}_t^{(0)}(w)\leq \ell_t(w),\; \forall w\in[-D,D]
\end{align}

\paragraph{Case 1: $w_t\geq 0$, $x_t\geq \frac{1}{D}$.}
For that case, the approximate loss $\tilde{\ell}_t$ of \cref{Equation:ApproximationLoss} can be written as follows:
\begin{equation}
\tilde{\ell}_t(w)=
\left\{
\begin{aligned} \label{Equation:TildeLoss}
\ell_t(w_t) + \ell_t'(w_t)(w-w_t);  \qquad \qquad      &\quad \text{if $w\in [0,D]$} \\
\ell_t(w_t) + \ell_t'(w_t)(w-w_t)+\frac{\beta}{2} y_t^2 w^2;        &\quad \text{if  $w\in [-D,0]$ }
\end{aligned} 
\right.
\end{equation}
where $\ell_t(w) = \log(1+e^{x_t w})$, $y_t  =\ell_t'(w_t)= \frac{x_t e^{x_t w_t}}{1+e^{x_t w_t}}$.
It is easily noticed that $\tilde{\ell}_t(w_t)=\ell_t(w_t)$. Also note that for positive instances $\tilde{\ell}_t(w)$ is the tangent of $\ell_t(w)$ at $w_t$, since $\ell_t(w)$ is convex it follows that:
$$\tilde{\ell}_t(w)\leq \ell_t(w),\quad \forall w\in[0,D]$$
We are left to prove the latter inequality holds for negative instances. Recalling $\tilde{\ell}_t^{(0)}$ from \cref{Equation:LosstildeZero}, we will show that:
\begin{align}\label{Equation:IneqTwoSided}
\tilde{\ell}_t(w)\leq \tilde{\ell}_t^{(0)}(w)\leq \ell_t(w) \qquad \forall w\in [-D,0]
\end{align}
Thus, concluding the proof. The lefthand inequality of \cref{Equation:IneqTwoSided} can be derived as follows: 
\begin{align}\label{Equation:IneqTwoSidedLeft}
\tilde{\ell}_t(w)&=\ell_t(w_t) + \ell_t'(w_t)(w-w_t)+\frac{\beta}{2} y_t^2 w^2  = \tilde{\ell}_t(0)+\ell_t'(w_t) w+\frac{\beta}{2} y_t^2 w^2 \\ \nonumber
&\leq  \ell_t(0)+\ell_t'(0) w+\frac{\beta}{2} x_t^2 w^2=\tilde{\ell}_t^{(0)}(w), \qquad \forall w\leq 0, w_t\in[0,D] \nonumber
\end{align}
where we used $\tilde{\ell}_t(0)\leq \ell_t(0)$, $0\leq \ell_t'(0)\leq \ell_t'(w_t)$, and $w\leq0$, moreover we used $|y_t|=  |\frac{x_t e^{x_t w_t}}{1+e^{x_t w_t}}|\leq |x_t|$.
The righthand inequality of \cref{Equation:IneqTwoSided}, is proved in the former case, see \cref{Equation:LossZerotildeInequality}.

The proof for the case $w_t\leq 0$, $x_t\leq -\frac{1}{D}$ is similar.

\paragraph{Case 2: $|x_t|\leq\frac{1}{D}$.}

For that case, the approximate loss $\tilde{\ell}_t$ of \cref{Equation:ApproximationLoss} can be written as follows:
\begin{equation}\label{Equation:LtildeSquare}
\tilde{\ell}_t(w)= \ell_t(w_t) + \ell_t'(w_t) (w-w_t) + \frac{\beta}{2}\big(\ell_t'(w_t)\big)^2(w-w_t)^2
\end{equation}
where we used, $y_t  =\ell_t'(w_t)$. Noticeably $\tilde{\ell}_t(w_t)=\ell_t(w_t)$. To prove $\tilde{\ell}_t(w)\leq \ell_t(w)$,
we require the following lemma from  \cite{HazanAK07}:
 \begin{lemma}\label{Lemma:ExpConcavity}
 For a function $f:\mathcal{K}\to R$, where $\mathcal{K}$ has diameter $D$, such that $\forall w\in \mathcal{K}$, $||\nabla f(w)||\leq G$, and $e^{-\alpha f(w)}$ is concave,
 the following holds for $\gamma =\frac{1}{2}\min\{\frac{1}{4GD},\alpha \}$:
 \begin{align*}
 f(w)\geq f(w_0) + \nabla f(w_0)^T(w-w_0) + \frac{\gamma}{2}(\nabla f(w_0)^T(w-w_0))^2, \qquad \forall w,w_0\in\mathcal{K}
 \end{align*}
 \end{lemma}
In  \cite{HazanAK07} it is also shown that for one dimensional functions, if $\alpha\leq \min_{w\in \mathcal{K}}\frac{f''(w)}{\big(f'(w)\big)^2}$, then $e^{-\alpha f(w)}$ is concave
in $\mathcal{K}$. In the case of logistic loss $\ell_t(w) = \log(1+e^{x_t w})$, the norm of its derivative is bounded by $1$, moreover:
\begin{align} \label{Equation:LogitExpConcavity}
\frac{\ell_t''(w)}{\big(\ell_t'(w)\big)^2}=e^{-x_t w}
\end{align}
Since $|x_t|\leq \frac{1}{D}$, and $|w|\leq D$, then $\alpha_0 :=e^{-1}\leq \min_{w\in[-D,D]}\frac{\ell_t''(w)}{\big(\ell_t'(w)\big)^2}$, implying
 $\gamma= \frac{1}{2}\min\{\frac{1}{4D},e^{-1} \}=\frac{1}{8D}$.
Applying \cref{Lemma:ExpConcavity} to the logistic loss $\ell_t(w)$, and $w_0=w_t$, we get:
\begin{align*}
\ell_t(w)\geq \ell_t(w_t) + \ell_t'(w_t)(w-w_t) + \frac{1}{2}\frac{1}{8D}( \ell_t'(w_t)(w-w_t))^2:=\tilde{\ell}_t(w), \qquad \forall w\in[-D,D]
\end{align*}
which proved the lemma for that case.

\paragraph{Case 3: $x_t w_t\leq 0$.}

Assume, without loss of generality, that $w_t>0$ and $x_t<0$. For that case, the approximate loss $\tilde{\ell}_t$ has the same form as in \cref{Equation:LtildeSquare}.
It is easily noticed that $\tilde{\ell}_t(w_t)=\ell_t(w_t)$. Notice that in $[0,D]$ we have:
$$e^{-x_t w} =e^{|x_t| w} \geq  \frac{1}{2} $$
where we used $x_t< 0$, and $w\in[0,D]$. According to \cref{Equation:LogitExpConcavity} it implies that $e^{-0.5 \ell_t(w)}$ is concave in $[0,D]$; applying \cref{Lemma:ExpConcavity}, we get:
\begin{align*}
\ell_t(w)\geq \ell_t(w_t) + \ell_t'(w_t)(w-w_t) + \frac{1}{2}\frac{1}{8D}( \ell_t'(w_t)(w-w_t))^2:=\tilde{\ell}_t(w), \qquad \forall w\in[0,D]
\end{align*}
and we used $\frac{1}{8D} = \frac{1}{2}\min\{\frac{1}{4D},\frac{1}{2} \}$.
Next we show that $\tilde{\ell}_t(w)\leq \tilde{\ell}_t^{(0)}(w),\; \forall w\in[-D,0]$, where $\tilde{\ell}_t^{(0)}$ is defined in \cref{Equation:LosstildeZero}.
Writing $\tilde{\ell}_t(w)$ we get:
\begin{align*}
\tilde{\ell}_t(w) &=  \ell_t(w_t) + \ell_t'(w_t)(w-w_t) + \frac{\beta}{2}( \ell_t'(w_t)(w-w_t))^2\\
 &=\tilde{\ell}_t(0) +\big( \frac{e^{x_t w_t}}{1+e^{x_t w_t}}(1-\beta \frac{x_t w_t e^{x_t w_t}}{1+e^{x_t w_t}})\big) x_t w + (\frac{e^{x_t w_t}}{1+e^{x_t w_t}})^2 \frac{\beta}{2}x_t^2 w^2
\end{align*}
where we used $\ell_t'(w) = x_t \frac{e^{x_t w_t}}{1+e^{x_t w_t}}$. Let's denote $z=x_t w_t <0$, and note that for $z\leq 0$, the  following holds:
\begin{align*}
\frac{e^{z}}{1+e^{z}} \leq \frac{1}{2}, \qquad \frac{e^{z}}{1+e^{z}}(1-\frac{z e^{z}}{1+e^{z}})\leq 1, \qquad \forall z\leq 0
\end{align*}
Using the latter expression for $\tilde{\ell}_t$, and the  two inequalities above:
\begin{align*}
\tilde{\ell}_t(w) 
 \leq \tilde{\ell}_t(0) +\frac{1}{2} x_t w + \frac{\beta}{2}x_t^2 w^2 \leq \ell_t(0) +\frac{1}{2} x_t w + \frac{\beta}{2}x_t^2 w^2: = \tilde{\ell}_t^{(0)}(w), \qquad \forall w\in[-D,0]
\end{align*}
where we used $z=x_t w_t\leq 0$, and $\tilde{\ell}_t(0)\leq \ell_t(0)$. Combining the latter inequality with \cref{Equation:LossZerotildeInequality}, proves:
\begin{align*}
\tilde{\ell}_t(w) \leq  \tilde{\ell}_t^{(0)}(w)\leq \ell_t(w), \qquad \forall w\in[-D,0]
\end{align*}
which concludes the proof.
\end{proof}

\section{Summary and Open Questions} \label{sec:summary}

We have given tight bounds for stochastic and online logistic regression that preclude the existence of fast rates for logistic regression without exponential factors. 
As a consequence, we have also resolved the COLT 2012 open problem of \cite{mcmahan2012open}.
Our lower bounds can be  extended to the multidimensional setting in which the instances are normalized and the labels are binary.

Our results suggest that second-order methods might present poor performance in practical logistic regression problems.
Indeed, in the derivation of our lower bounds we have constructed a distribution over instances such that the induced expected loss function is approximately linear around its optimum.

An interesting feature of our  results is that our regret/convergence bounds apply to \emph{a finite range of $T$}, and are different than the known asymptotic bounds. 
Arguably, the range of $T$ for which our results apply is the important one in practice (sub-exponential in the size of the hypothesis class). Are there other natural settings in which regret bounds for bounded number of iterations differ from the asymptotic bound?

\bibliographystyle{abbrvnat}
\bibliography{bib}

\newpage
\appendix

\section{Proof of \cref{thm:coin}}\label{appendixA:InformationTheoretic}
Suppose a randomize algorithm $\A$ that given $m$ tosses decides upon one of the coins,  
and denote by $\D_{\A}$ the conditional distribution of the algorithm over his decision given the $m$ coin tosses.
We also let $\D_p$, $\D_{p+\eps}$ denote the respective Bernoulli distributions corresponding to a single toss; let $\D_p^m$, $\D_{p+\eps}^m$ be the product distributions of a sequence of $m$ independent tosses, and let $\D^m_{p,\A}$, $\D^m_{{p+\eps},\A}$ be the joint distributions over the sequence of $m$ independent tosses and the decision of the randomized algorithm.
For the proof we need the following standard lemma.
\begin{lemma} \label{lem:pinsker}
For all events $B$ in the space of $m$ independent tosses and the decision of the algorithm:
\begin{align*}
	\abs{\D^m_{p,\A}(B) - \D^m_{p+\eps,\A}(B)}
	\le \sqrt{\frac{m \eps^2}{p}} ~.
\end{align*}
\end{lemma}
\begin{proof}
We first bound the KL-divergence between $\D_p$ and $\D_{p+\eps}$.
Using the fact $\log z \le z-1$ for $z > 0$, we obtain
\begin{align*}
	\KL(\D_{p+\eps} \mid\mid  \D_p ) 
	&= (p +\eps)\log\frac{p+\eps}{p} + (1-p-\eps) \log\frac{1-p-\eps}{1-p} \\
	&\le (p +\eps) \lr{\frac{p+\eps}{p}-1} + (1-p-\eps) \lr{\frac{1-p-\eps}{1-p}-1}\\
	&= \frac{\eps^2}{p(1-p)} ~.
\end{align*}
Since the decision of the algorithm only depends on the $m$ tosses that $\A$ observes, we may write:
\begin{align} \label{equation:KLassist}
\D^m_{p,\A}=\D^m_{p}\D_{\A},\qquad \D^m_{p+\eps,\A}=\D^m_{p+\eps}\D_{\A}
\end{align}
Thus, we can write:
\begin{align*}
\KL(\D^m_{p+\eps,\A} \mid\mid \D^m_{p,\A})= 
	\KL(\D_{p+\eps}^m \mid\mid \D_{p}^m) 
	= m \KL(\D_{p+\eps} \mid\mid \D_{p}) 
	\le \frac{m\eps^2}{p(1-p)} 
\end{align*}
the first equality follows from \cref{equation:KLassist} combined with the definition of the KL-divergence, the second equality holds since the KL-divergence is additive over distribution products.
Finally, recalling Pinsker's inequality we conclude that for all events $B$ in the joint space of tosses and algorithm's decision:
\begin{align*}
	\abs{\D^m_{p,\A}(B) - \D^m_{p+\eps,\A}(B)}
	\le \sqrt{\tfrac{1}{2} \KL(\D^m_{p+\eps,\A} \mid\mid \D^m_{p,\A})}
	\le \sqrt{\frac{m \eps^2}{2p(1-p)}}
	\le \sqrt{\frac{m \eps^2}{p}} \,.
\end{align*}
\end{proof}
where in the last inequality we used $p\in(0,\half]$. We can now prove \cref{thm:coin}.
\begin{proof}
Having an algorithm $\A$ that discovers the correct coin w.p$\geq 3/4$,  let $B$ be the event that the algorithm decides that nature uses the first coin after $m$ tosses, then clearly:
\begin{align*}
 \abs{\D^m_{p,\A}(B) - \D^m_{p+\eps,\A}(B)}
 \geq 1/4
\end{align*}
combining the latter with \cref{lem:pinsker} proves \cref{thm:coin}.
\end{proof}

\end{document}